\documentclass{article}

\usepackage[preprint]{neurips_2022}

\usepackage[utf8]{inputenc} % allow utf-8 input
\usepackage[T1]{fontenc}    % use 8-bit T1 fonts
\usepackage{hyperref}       % hyperlinks
\usepackage{url}            % simple URL typesetting
\usepackage{booktabs}       % professional-quality tables
\usepackage{amsfonts}       % blackboard math symbols
\usepackage{nicefrac}       % compact symbols for 1/2, etc.
\usepackage{microtype}      % microtypography
\usepackage{xcolor}         % colors

\usepackage{amssymb,amsthm,amsfonts,amscd,mathrsfs}
\usepackage{thm-restate}
\usepackage{amsmath}
\usepackage{xspace}
\usepackage{graphicx}
\usepackage{multirow}

\newcommand{\algoname}{{\AE}qprop\xspace}

\newtheorem{thm}{Theorem}
\newtheorem{lma}[thm]{Lemma}
\newtheorem{prop}[thm]{Proposition}
\newtheorem{cor}[thm]{Corollary}
\newtheorem{hyp}[thm]{Assumption}
\newtheorem{defi}[thm]{Definition}

\newcommand{\E}{\mathbb{E}}
\newcommand{\R}{\mathbb{R}}
\newcommand{\eps}{\varepsilon}
\newcommand{\deq}{\mathrel{\mathop{:}}=}

\newcommand{\abs}[1]{\left\lvert#1\right\rvert}
\newcommand{\norm}[1]{\left\lVert#1\right\rVert}
\def\d{\operatorname{d}\!{}}
\newcommand{\lyap}{\mathcal{L}}

\DeclareMathOperator*{\argmin}{arg\,min}
\DeclareMathOperator{\Id}{Id}
\newcommand{\grad}{\nabla\!}
\newcommand{\from}{\colon} % correct ':' in f\from X \to Y

\newcommand{\plh}{{\mkern-2mu\times\mkern-2mu}} % \times without space

\newcommand{\TODO}[1]{{\color{red} TODO: {#1}}}
\newcommand{\todo}[1]{\TODO{#1}}
\newcommand{\authorcomment}[2]{{\color[rgb]{#1}#2}}
\newcommand{\NDY}[1]{\authorcomment{0.0,0.8,0.4}{[NdY: #1]}}
\renewcommand{\TODO}[1]{}
\renewcommand{\todo}[1]{}
\renewcommand{\authorcomment}[2]{}
\renewcommand{\NDY}[1]{}

\newcommand{\neuripsonly}[1]{}

\newtheorem{rem}[thm]{Remark}

\title{Agnostic Physics-Driven Deep Learning}

\author{
  Benjamin Scellier \\
  SAM, D-MATH, ETH Zurich\\
  \And
  Siddhartha Mishra \\
  SAM, D-MATH and AI Center, ETH Zurich \\
  \AND
  Yoshua Bengio \\
  Mila, University of Montreal \\
  \And
  Yann Ollivier \\
  Facebook A.I. Research, Paris \\
}

\begin{document}

\maketitle

\begin{abstract}
This work establishes that a physical system can perform statistical
learning without gradient computations, via an \emph{Agnostic Equilibrium
Propagation} (\algoname) procedure that combines energy
minimization, homeostatic control, and nudging towards the correct
response. 
In \algoname, the specifics of the system do not have to be known: the
procedure is based only on external manipulations, and produces a
stochastic gradient descent without explicit gradient computations.
Thanks to nudging, the system performs a true, order-one
gradient step for each training sample, in contrast with order-zero
methods like reinforcement or evolutionary strategies, which rely on
trial and error. This procedure considerably widens the range of
potential hardware for statistical learning to any system with enough
controllable parameters, even if the details of the system are poorly
known. \algoname also establishes that in natural (bio)physical systems,
genuine gradient-based statistical learning may result from
generic, relatively simple mechanisms, without backpropagation and its requirement for analytic knowledge of partial derivatives.
\todo{Is that correct? one could argue this was already known eg Hebbian
mechanisms...}
\end{abstract}

\section{Introduction}
\label{sec:introduction}

In the last decade, deep learning has emerged as the leading approach to machine learning \citep{lecun2015deep}.
Deep neural networks have significantly improved the state of the art in pretty much all domains of artificial intelligence.
However, as neural networks get scaled up further, training and running them on Graphics Processing Units (GPUs) becomes slow and energy intensive.
These inefficiencies can be attributed to the so-called \textit{von Neumann bottleneck} i.e., the separation of processing and memory creating a bottleneck for the flow of information.
Considerable efficiency gains would be possible by rethinking hardware for machine learning, taking inspiration from the brain and other biological/physical systems where processing and memory are two sides of the same physical unit. \todo{ref}

One of the central tools of deep learning is optimization by gradient
descent, usually performed by the backpropagation algorithm.
Works such as \cite{wright2022deep} establish that various physical
systems can perform machine learning computations efficiently for
inference; still, 
gradient training is done \emph{in silico} on a digital model of the
system.
We believe that building truly efficient hardware for large-scale gradient-descent-based machine learning also requires rethinking the learning algorithms to be better integrated within the underlying system's physical laws.

Equilibrium propagation (Eqprop) is an alternative mathematical framework
for gradient-descent-based machine learning, in which inference and gradient computations are both performed using the same physical laws
\citep{scellier2017equilibrium}.
In principle, this offers the possibility to optimize arbitrary physical systems and loss functions by gradient descent \citep{scellier2020deep}. Eqprop applies, in particular, to physical systems whose equilibrium state minimizes an energy function,
e.g. nonlinear resistive networks \citep{kendall2020training}.
Such physical networks may be called `energy-based models' in the machine learning terminology, but energy minimization in these networks is directly performed by the laws of physics (not with numerical methods in a computer simulation).
In Eqprop, the gradient of the loss function is computed with two
measurements.
In a first phase, the system settles to equilibrium after presenting an input.
In a second phase, the energy of the system is slightly modified so as to
nudge the output towards a desired response, and the system settles to a
new equilibrium. The gradient is estimated from these two equilibrium states--- see Appendix~\ref{sec:eqprop} for more details on Eqprop. This approach has already been physically realized, e.g., in
\cite{dillavou2021demonstration} using a small variable resistor electrical network.

However, three challenges remain for training physical systems with
Eqprop. First, part of the analytical form of the
energy function of the system must be known explicitly (all partial
derivatives of the energy function with respect to the parameters).
Second---and most importantly---once gradients have been computed,
the trainable parameters still need to be physically updated by some
(nontrivial) physical procedure. Many articles propose to store parameters
as conductance values of non-volatile memory (NVM) elements (e.g.
memristors \citep{chua1971memristor}), but these NVM elements are far
from ideal and updating them continues to be extremely challenging \citep{chang2017mitigating}.
Third, the equilibrium state of the first phase of Eqprop needs to be stored somehow, for later use in the gradient computation.

We introduce \textit{Agnostic Eqprop} (\algoname), a novel alternative to Eqprop that overcomes these three challenges in one stroke. 
\algoname exploits the underlying physics of the system not just to perform the computations at inference, but also to physically adjust the system's parameters in proportion to their gradients.
To achieve this, in \algoname, the parameter variables are seen as floating
variables that also minimize the energy of the system, just like the
state variables do. We also require that each parameter variable is
coupled to a \emph{control knob} that can be used to maintain the
parameter around its current value while the system settles.

In \algoname as in Eqprop, training consists of iterating
over two phases for each training sample,
with a modified energy in the second phase:
\begin{enumerate}
\item In the first phase (inference), the input variables are set to some value; the output and state variables are allowed to
evolve freely, whereas the control knob variables are set so that the trainable parameters remain fixed. 
\item In the second phase, the inputs and controls are fixed at the values of the first phase, and the output is
slightly pushed (or `nudged') towards the desired value for the input by acting on the
underlying output energy function; the state and parameters are allowed to evolve freely, and this slightly moves the parameters towards a new value.
\end{enumerate}

After iterating over many
examples, the parameters evolve so that the output spontaneously
produces an approximation of the desired value. Indeed, we prove that the parameter change in the second phase corresponds to one step of
gradient descent with respect to the loss function
(Theorem~\ref{thm:sgd}). We also show that \algoname has some better
performance guarantees than gradient descent, especially in the so-called
\emph{Pessimistic} variant of \algoname: contrary to gradient descent,
even with large step sizes,
each step of \algoname is guaranteed to reduce a tight bound on the loss
function, evaluated on the example used at
that step (Theorem~\ref{thm:lyapunov}).

In this process, \algoname is agnostic to the analytical form of the
energy function, and there is no need to store the first equilibrium state.
Although no gradients are computed explicitly, \algoname is a first-order
method, not a zero-order method like evolutionary strategies: at each
step, the parameters do follow the gradient of the error on the given
sample.

\section{\algoname: an Agnostic Physical Procedure for Gradient Descent}
\label{sec:aeqprop}

We consider a prototypical machine learning problem: minimize an objective function
\begin{equation}
    \label{eq:loss}
J(\theta)=\E_{(x,y)} \, \left[ C(s(\theta,x),y) \right]
\end{equation}
over some parameter $\theta$, where the variable $x$ represents some inputs\footnote{All quantities in this text are vectors, not scalars, unless otherwise specified.}, the variable $y$ represents desired outputs, $C$ is a cost function, and $s$ is some quantity computed by the system from $\theta$ and $x$, that encodes a prediction with respect to $y$. The expectation represents the distribution of values we want to predict.

In machine learning, the workhorse for this problem is stochastic
gradient descent (SGD) \citep{
%robbins1951stochastic,
bottou2010large},
\begin{equation}
\label{eq:sgd}
\theta_{t}=\theta_{t-1}-\eta_t \,\partial_\theta C(s(\theta_{t-1},x_t),y_t)
\end{equation}
with step size (learning rate) $\eta_t$, where at each step, an example $(x_t,y_t)$ is
chosen at random from a training set of examples. (With batching, each
$x_t$ and $y_t$ may represent a set of several inputs and desired outputs.)

Here, following \cite{scellier2017equilibrium}, we assume that the function $s(\theta,x)$ is obtained by a physical
process that minimizes some energy function $E$,
\begin{equation}
s(\theta,x)=\argmin_s E(\theta,x,s).
\end{equation}
Namely, we use physical equilibration of the system as the computing
process. Many physical systems evolve by minimizing some quantity \citep{millar1951cxvi,cherry1951cxvii,wyatt1989criteria, kendall2020training, stern2021supervised}
\todo{Better REF ?}, so we take this equilibration as the basic computational step. \footnote{The function $E$ does not have to be the physical energy of the system:
it may be any function effectively minimized by the system's spontaneous
evolution. For instance, in a thermodynamical system, $E$ may be the free
energy.} Below, we will also assume that the parameter $\theta$ itself
is a part of this
system and follows the energy minimization to change during equilibration. 

\algoname is a physical procedure that allows an operator (running the computing system) to simulate 
stochastic gradient descent \eqref{eq:sgd} by pure physical manipulations, \emph{without explicitly knowing the energy function $E$} or other details of the system.

We assume that this operator has the following abilities:
\begin{itemize}
\item The ability to clamp (set) part of the state, the ``input knobs'', to any desired value $x$, then let the system reach equilibrium, and read the system's response on some part of the state $s$, the ``output unit''.
\item The ability to \emph{nudge} the system towards any desired output $y$, by adding $\beta C(s,y)$ to the energy of the system, where $\beta>0$ is a small constant. This requires knowledge of the cost function: for instance, adding a small quadratic coupling between the output unit and the desired output $y$ to minimize the squared prediction error.
\item The ability to control the parameters $\theta$ via
\emph{control knobs} $u$, thanks to a strong (but not infinite) coupling
energy, e.g. $\norm{u-\theta}^2/2\eps$ with small $\eps$. Requiring one control
knob per parameter, the operator needs to adjust $u$ in real time while the system evolves so that $\theta$ remains at a constant
value (\emph{homeostatic control} of $\theta$ by $u$). The operator also requires the ability to
\emph{clamp} $u$ to its current value.
\end{itemize}

So at each instant, we set input knobs $x$, control knobs $u$, and possibly (if $\beta>0$) a desired output $y$, and assume that the system reaches an equilibrium $(\theta_\star,s_\star) = \argmin_{(\theta,s)} \; \mathcal{E}(u,\theta,s,x,y,\eps,\beta)$, where
\begin{equation}
\label{eq:fullenergy}
\mathcal{E}(u,\theta,s,x,y,\eps,\beta) :=
\norm{u-\theta}^2/2\eps+E(\theta,x,s)+\beta \, C(s,y)
\end{equation}
is the global energy function of the system. In the default formulation of \algoname (the so-called \textit{Optimistic} variant), we will set $\beta$ to two values only: $0$ and a small positive value.

The energy function $E(\theta,x,s)$ need not be known explicitly, but must be complex enough that we
can make the system reach any desired behavior by adjusting the parameter $\theta$.

\paragraph{The \algoname procedure.}
Under these assumptions, the following procedure simulates gradient
descent in the physical system.
\begin{enumerate}
\item Observe the current value $\theta_{t-1}$ of the parameter.
\item \label{item:freestep} Present the next input example $x_t$ to the system, without nudging
($\beta=0$). Let the system reach equilibrium, while at the same time,
adjusting the control knobs $u$ so that the parameter $\theta$ remains at $\theta_{t-1}$.
\item Clamp the control knobs to their current value $u_t$. Turn on the
nudging to the desired output $y_t$ by adding $\beta C(s,y_t)$ to the energy
function of the system, where $\beta > 0$.
\item Let the system reach a new equilibrium for $s$ and $\theta$ given
$u_t$, $x_t$, $y_t$ and $\beta$. Read
the new value $\theta_{t}$ of the parameter.
\end{enumerate}
In formulae, this means that we first set a control value $u_{t}$ such
that the equilibrium value $\theta_{t-1}$ does not change when we
introduce the new input $x_t$:
\begin{align}
\label{eq:controlstep}
\text{set }u_t\text{ such that } \quad \theta_{t-1}&=\argmin_\theta \; \min_s \; \mathcal{E}(u_{t},\theta,s,x_t,y_t,\eps,0)
\intertext{and then obtain the next parameter by adding some nudging $\beta$ and letting the system reach
equilibrium,}
\label{eq:nudgestep}
\theta_{t}&=\argmin_{\theta} \; \min_s \; \mathcal{E}(u_{t},\theta,s,x_t,y_t,\eps,\beta).
\end{align}

The above loop is repeated over all pairs $(x_t,y_t)$ in the training set.
After training, the system can be used for inference, without nudging. Hence, only step
\ref{item:freestep} is used: set the
input knobs to some input $x$ while adjusting the controls $u$ so
$\theta$ does not change, let the system relax to equilibrium, then read
the output variable.
Alternatively, after training,
the parameters $\theta$ can just be clamped to their
final value, which avoids the need for further homeostatic control via
$u$.

Next, we show the following outcome of the \algoname procedure.
\begin{thm}
\label{thm:sgd}
Under technical assumptions, for small $\eps$ and $\beta$ we have
\begin{equation}
\theta_{t}=\theta_{t-1}-\eps\beta\, \partial_\theta C(s(\theta_{t-1},x_t),y_t)
+O(\eps^2\beta+\eps\beta^2).
\end{equation}
\end{thm}

Namely, the \algoname procedure performs a step of stochastic gradient descent for
the input-output pair $(x_t,y_t)$,
with step size $\eps\beta$. Note that
neither the energy function, nor its gradients, nor the gradients of the
cost function have been used.

A proof of Theorem \ref{thm:sgd} is provided in Appendix
\ref{sec:proofs}. Appendix \ref{sec:riemannian-sgd} also describes
extensions to situations where only one of $\eps$ or $\beta$ is small, to situations where the coupling between $u$ and $\theta$ is not of the
form $\norm{u-\theta}^2 / 2$ (resulting in a \emph{Riemannian} SGD ; for instance, using a per-component coupling $\sum_k
(u_k-\theta_k)^2/2\eps_k$ results in per-component step sizes $\eps_k \, \beta$), 
and gives more details on the $O(\eps^2\beta+\eps\beta^2)$
term.

\textbf{Remark.}
%\todo{Is it the best place for this remark?}
It is well-known that gradient descent can be approximately written as
minimizing a cost function penalized by the distance to the previous
value,
\begin{equation}
\argmin_\theta\{C(s(\theta,x),y)+\norm{\theta-\theta_{t-1}}^2/2\eps\}
\approx \theta_{t-1}-\eps\, \partial_\theta C(s(\theta_{t-1},x),y).
\end{equation}
So it might seem that we just have to set $u=\theta_{t-1}$ and add the
energy function $C(s,y)$. However, as soon as we add $C(s,y)$ to the
energy, we change the equilibrium value for $s$, so that $s\neq
s(\theta_{t-1},x)$ anymore. Likewise, presenting the input $x_t$ with a
fixed $u$ will change $\theta$. This is why we have to use a more
complicated procedure in \algoname.

We now turn to an important aspect of \algoname's behavior when $\eps$ and $\beta$ are not infinitesimal: the existence of a Lyapunov function.

\section{Monotonous Improvement: A Lyapunov Function for \algoname}

Let us rewrite the objective function \eqref{eq:loss} in the form:
\begin{equation}
    J(\theta)=\E_{(x,y)} \, \left[ \lyap(\theta,x,y) \right], \qquad \text{where} \qquad \lyap(\theta,x,y) \deq C(s(\theta,x),y).
\end{equation}
We call $\lyap$ the ``loss function'', to distinguish it from the objective function ($J$) and the cost function ($C$).

Theorem \ref{thm:sgd}
%, which rewrites $\theta_{t}=\theta_{t-1}-\eps\beta\, \partial_\theta \lyap(\theta_{t-1},x_t,y_t) + O(\eps^2\beta+\eps\beta^2)$,
holds in the regime of infinitesimal step sizes $\beta$ and $\epsilon$, but what if $\beta$ and/or $\eps$ are non-infinitesimal?
It is in this context of non-infinitesimal $\beta,\eps$ that \algoname has some better theoretical properties than stochastic
gradient descent (SGD). In SGD with predefined step size, there is
no guarantee that the gradient step will improve the output, unless some
a priori information is available such as bounds on the Hessian of the
loss function. \todo{refs}

On the other hand, in \algoname, there exists a
\emph{Lyapunov function} for each step of the procedure, even when $\eps$ and $\beta$ are
nonzero. More precisely, there exists a function
$\lyap_\beta(\theta,x,y)$ such that
\begin{itemize}
\item $\lyap_\beta(\theta,x,y)\to \lyap(\theta,x,y)$ when $\beta\to 0$,
namely, $\lyap_\beta$ is close to the true loss when $\beta$ is small;
\item $\lyap_\beta(\theta_{t},x_t,y_t)\leq
\lyap_\beta(\theta_{t-1},x_t,y_t)$ for any choice of $\eps$ and $\beta$.
\end{itemize}

The above property is essential for numerical stability: even though
$\lyap_\beta$ is not exactly $\lyap$ for $\beta\neq 0$, it means that the process is still minimizing a function close to $\lyap$, therefore it cannot diverge severely. We point out that the Lyapunov function depends on the current example $(x_t,y_t)$. Hence, performance improves on the current example only. For comparison, standard SGD does not satisfy even this property.

The Lyapunov property may be most interesting in the regime of large batch sizes, where each $x_t$ actually encodes a large number of training samples. In this regime, if each batch is sufficiently representative of the whole training set, then the Lyapunov function depends much less on the batch, and it serves as a proxy for the objective function $J(\theta)$. Denoting $J_\beta(\theta) \deq \E_{(x,y)} \, \left[ \lyap_\beta(\theta,x,y) \right]$, this leads to monotonous improvement along the learning procedure: $J_\beta(\theta_0) \geq J_\beta(\theta_1) \geq \ldots \geq J_\beta(\theta_t)$.

We now define the Lyapunov function $\lyap_\beta$, which is
closely related to the loss function $\lyap$.

\begin{restatable}{thm}{propdecreasing}
\label{thm:lyapunov}
For each $\beta>0$, let $s_\beta(\theta,x,y)$ be the state of the system with nudging $\beta$, i.e.
\begin{equation}
s_\beta(\theta,x,y)=\argmin_s \left\{E(\theta,x,s)+\beta C(s,y)\right\}.
\end{equation}
Define the Lyapunov function
\begin{equation}
\label{eq:deflyapunov}
\lyap_\beta(\theta,x,y)\deq \frac1\beta \int_{\beta^\prime=0}^\beta
C(s_{\beta^\prime}(\theta,x,y),y)\d \beta^\prime
\end{equation}
and note that $\lyap_\beta(\theta,x,y)\to \lyap(\theta,x,y)$ when
$\beta\to 0$.

Then for any $\beta>0$, 
along the \algoname trajectory $(\theta_t)$ given by
\eqref{eq:controlstep}--\eqref{eq:nudgestep},
we have
\begin{equation}
\lyap_\beta (\theta_t,x_t,y_t)\leq \lyap_\beta (\theta_{t-1},x_t,y_t).
\end{equation}
\end{restatable}

We prove Theorem \ref{thm:lyapunov} in Appendix
\ref{sec:proofs}. We emphasize that Theorem \ref{thm:lyapunov} holds for any value of $\beta > 0$, even far from the regime $\beta \to 0$, and regardless of $\eps$.

\section{Optimistic \algoname, Pessimistic \algoname, and Centered \algoname}
\label{sec:variants}

The Lyapunov function expression
\eqref{eq:deflyapunov}
shows that \algoname is slightly too
\emph{optimistic} at first order in $\beta$ as \algoname minimizes the underlying error
assuming that there
will be some nudging $\beta^\prime>0$ at test time.
This Lyapunov function also appears as the gradient
actually computed by \algoname when $\beta$ is fixed instead of $\beta\to 0$: \algoname really
has $\lyap_\beta$ as its loss function (Appendix \ref{sec:riemannian-sgd}, Theorem~\ref{thm:riemannian-sgd}).

It is possible to partially compensate or even reverse this effect. This
leads to \emph{centered} \algoname and \emph{pessimistic} \algoname: 
\begin{itemize}
\item In
unmodified (optimistic) \algoname, we use $\beta=0$ in the first step
\eqref{eq:controlstep} and positive
$\beta$ in the second step \eqref{eq:nudgestep}.
\item \emph{Pessimistic} \algoname uses $-\beta$ instead of $0$ in the
step \eqref{eq:controlstep}, and $0$ instead of $\beta$ in the step
\eqref{eq:nudgestep}. This amounts to assuming that there will be some
nudging \emph{against} the correct answer at test time.
\item \emph{Centered} \algoname uses $-\beta/2$ in step
\eqref{eq:controlstep} and $\beta/2$ in step \eqref{eq:nudgestep}. With
this, the resulting Lyapunov function is $O(\beta^2)$-close to the
loss function $\lyap$, instead of $O(\beta)$.
\end{itemize}

These variants enjoy similar theorems (Appendix~\ref{sec:riemannian-sgd}), and are tested below (Section \ref{sec:numerical-illustration}). In particular, the Lyapunov function $\lyap_{-\beta}$ for Pessimistic \algoname involves
an integral of $\beta^\prime$ from $-\beta$ to $0$ instead of $0$ to
$\beta$ in \eqref{eq:deflyapunov}: namely, it optimizes
under an assumption of \emph{negative} (adversarial) nudging at test time. 
Likewise,
Centered \algoname assumes a mixture
of positive and negative nudging at test time. Speculatively, this
might
improve robustness.

The Lyapunov functions for optimistic and pessimistic \algoname bound the
loss function for each sample (Appendix,
Theorem~\ref{thm:riemannian-sgd}):
\begin{equation}
\lyap_\beta(\theta,x,y)\leq \lyap(\theta,x,y)\leq
\lyap_{-\beta}(\theta,x,y).
\end{equation}
In particular, Pessimistic \algoname actually
optimizes an \emph{upper bound} of the true loss function $\lyap$ for
each sample.

Numerically, negative values of the nudging parameter $\beta$ require more care because a
negative term $-\abs{\beta}C(s,y)$ will be introduced to the energy: if $C$ is
unbounded (such as a quadratic cost), the equilibrium might be when $s\to\infty$ with the energy
tending to $-\infty$. This can be corrected by ensuring the main energy
$E(\theta,x,s)$ is sufficiently large for large $s$, for instance, for
a quadratic cost, by ensuring the
energy $E$ has an
$\norm{s}^2$-like term. \footnote{This is
slightly different from
parameter regularization in machine learning:
regularizing $E$ regularizes the model and the
state $s(\theta,x)$, but \algoname computes the unregularized
gradient of the same, unchanged loss function applied to the regularized model, instead of a
regularized gradient of the original model.}

\section{A Numerical Illustration}
\label{sec:numerical-illustration}

Computationally, there is little interest in a numerical simulation of
\algoname: this amounts to using a computer to simulate a physical system
that is supposed to emulate a computer, which is inefficient.
This is all the more true as the fundamental step of \algoname is energy
minimization, which we will simulate by gradient descent on the energy,
while stochastic gradient descent was the operation we wanted to simulate
in the first place.

Still, such a simulation is a sanity check of \algoname. We can 
compare \algoname with direct stochastic gradient descent, and
observe
the influence of the second-order terms (testing the influence of finite
$\beta>0$ instead of $\beta\to 0$). This also demonstrates that the energy minimization
and the homeostatic control of $\theta$ can be realized in a simple
%and numerically stable 
way, and that imperfect energy minimization does not
necessarily lead to unstable behavior.

We present two series of experiments: a simple linear regression example
(Section~\ref{sec:linreg}, and dense and convolutional Hopfield-like
networks on the real datasets MNIST and FashionMNIST
(Section~\ref{sec:hopfield}).

We start with a discussion of one possible, generic way to simulate the
energy minimization and homeostatic control numerically
(Section~\ref{sec:simucontrol}): an energy relaxation by gradient
descent, and a proportional controller on $u$. This is the implementation
used for the linear regression example.

However, with Hopfield networks on real datasets, such an explicit
physical simulation of energy minimization and homeostatis was
slow. We had to use algebraic knowledge to speed up the simulations:
for energy minimization, we
iteratively minimized each layer given the others (a 1D quadratic
minimization problem for each variable), and the control $u$ was
directly set to the algebraically computed correct value.

% As already mentioned, simulating a physical system that implements machine
% learning is expected to be inefficient. Simulating a gradient descent
% on the loss function by using an auxiliary gradient descent for energy minimization
% is not computationally interesting.  So the
% interest of these simulations is to observe the behavior of the \algoname
% procedure. In an actual physical system,
% energy minimization would be the system's spontaneous behavior.

\subsection{Simulating Convergence to Equilibrium and Homeostatic
Control}
\label{sec:simucontrol}

In the free (non-nudged, $\beta=0$) phase of \algoname, we have to fix the inputs to
$x_t$ and let the system $(s,\theta)$ relax to equilibrium, while at the
same time adjusting $u$ to ensure that the equilibrium value of $\theta$
is equal to the previous value $\theta_{t-1}$. Numerically, we realize
this by iterating a gradient descent step on the energy
\eqref{eq:fullenergy} of $(s,\theta)$. In parallel, we
implement a simple proportional controller on $u$, which increases $u$
when $\theta$ is too small:
\begin{align}
\label{eq:sstep}
s&\gets s-\eta_s \grad_s \mathcal{E}(u,\theta,s,x,y,0)=s-\eta_s
\grad_s E(\theta,x,s)
\\
\label{eq:thetastep}
\theta &\gets \theta - \eta_\theta \grad_\theta
\mathcal{E}(u,\theta,s,x,y,0)=\theta+\eta_\theta\frac{u-\theta}{\eps} -\eta_\theta
\grad_\theta E(\theta,x,s)
\\
\label{eq:ustep}
u &\gets u + \eta_u (\theta_{t-1}-\theta)
\end{align}
with respective step sizes $\eta_s$, $\eta_\theta$ and $\eta_u$. We
always use
\begin{equation}
\eta_u=\eta_\theta/(4\eps)
\end{equation}
which corresponds to the critically damped regime \footnote{Namely, the
linearized system on $(\theta,u)$ (without $s$) around its equilibrium value
$\theta=u=\theta_{t-1}$ has all
eigenvalues equal to $-1/2\eps$, which provides quickest convergence
without oscillations.} for the pair
$(\theta,u)$ and the coupling energy $U(u,\theta)=\norm{\theta-u}^2/2\eps$.

In practice, the step sizes $\eta_s$ and $\eta_u$ are adjusted adaptively
to guarantee that $\mathcal{E}$ decreases: first, a step \eqref{eq:sstep} is
applied, and if $\mathcal{E}$ decreases the step is accepted and $\eta_s$
is increased by $5\%$; if $\mathcal{E}$ increases the step is cancelled
and $\eta_s$ is decreased by $50\%$; if $\mathcal{E}$ is unchanged we
either multiply or divied $\eta_s$ by $1.05$ with probability $1/2$.
Then, the same is applied for the step \eqref{eq:thetastep} on $\theta$.
Finally, if the step on $\theta$ was accepted then we perform a step
\eqref{eq:ustep} on $u$, with step size $\eta_u=\eta_\theta/(4\eps)$.
Then we loop over
\eqref{eq:sstep}--\eqref{eq:thetastep}--\eqref{eq:ustep} again. The
step sizes were initialized to $\eta_s=1$ and $\eta_\theta=\eps$.

For the nudged step of \algoname $(\beta\neq 0)$, we apply the same principles, but
with $u$ fixed $(\eta_u=0)$, and with $\mathcal{E}$ evaluated at $\beta$
instead of $0$: this results in an additional term $-\eta_s \beta \grad_s
C(s,y)$ in \eqref{eq:sstep} for the update of $s$.

In our experiments, convergence to equilibrium was simulated by iterating
$50$ steps of these updates.

For the controller, we could also directly set $u$ to the optimal value
$u^\ast=\theta_{t-1}+\eps \grad_\theta
E(\theta_{t-1},x,s(\theta_{t-1},x_t))$, which
guarantees an equilibrium at $\theta=\theta_{t-1}$. However, we do not
consider this a realistic scenario for \algoname: contrary to $s$ and
$\theta$ which evolve spontaneously, $u$ must be set by an external
operator, and this operator may not have access to the energy
function $E$ or its gradient. The controller \eqref{eq:ustep} just uses
$\theta_{t-1}$ and the observed $\theta$.

\subsection{A Simple Linear Regression Example}
\label{sec:linreg}

For this experiment we consider linear regression on $[-1;1]$. Let
$f\from [-1;1]\to \R$ be
a target function. Let $\phi_1,\ldots,\phi_k\from [-1;1]\to \R$ be $k$
feature functions. The model to be learned is $f(x)\approx\sum_i \theta_i
\phi_i(x)$. In this section the features $\phi_i$ are fixed,
corresponding to a linear model.

We are going to apply \algoname with parameter $\theta=(\theta_i)$, input
$x\in [-1;1]$ and output $y=f(x)$ for random samples $x\in [-1;1]$. The
state is a single number $s\in \R$, and the energy and cost are
\begin{equation}
E(\theta,x,s)=\frac12 \left(s-\sum_i \theta_i \phi_i(x)\right)^2,\qquad
C(s,y)=\frac12 (s-y)^2.
\end{equation}
In the free phase ($\beta=0$), the system relaxes to $s=\sum_i \theta_i
\phi_i(x)$.

The features $\phi_i$ were taken to be the Fourier features
$\phi_1(z)=1$,
$\phi_{2i}(z)=\sin(i\pi z)$, $\phi_{2i+1}(z)=\cos(i\pi z)$, up to
frequency $i=10$. The ground truth
function $f$ is a random polynomial of degree $d=10$, defined as
\begin{equation}
\label{eq:truefunc}
f(z)\deq \sum_{i=0}^d w_i L_i(z)
\end{equation}
where $L_i(z)$ is the Legendre polynomial of degree $i$, and where the $w_i$
are independent Gaussian random variables $N(0,1)$. Thanks to
the Legendre polynomials being orthogonal, this model produces random
polynomials $f$ with a nice range; see an example in
Fig.~\ref{fig:truefuncsample}. Since
we use Fourier features while $f$ is a polynomial (and non-periodic),
there is no exact solution.

\begin{figure}
\begin{center}
\includegraphics[width=.35\textwidth]{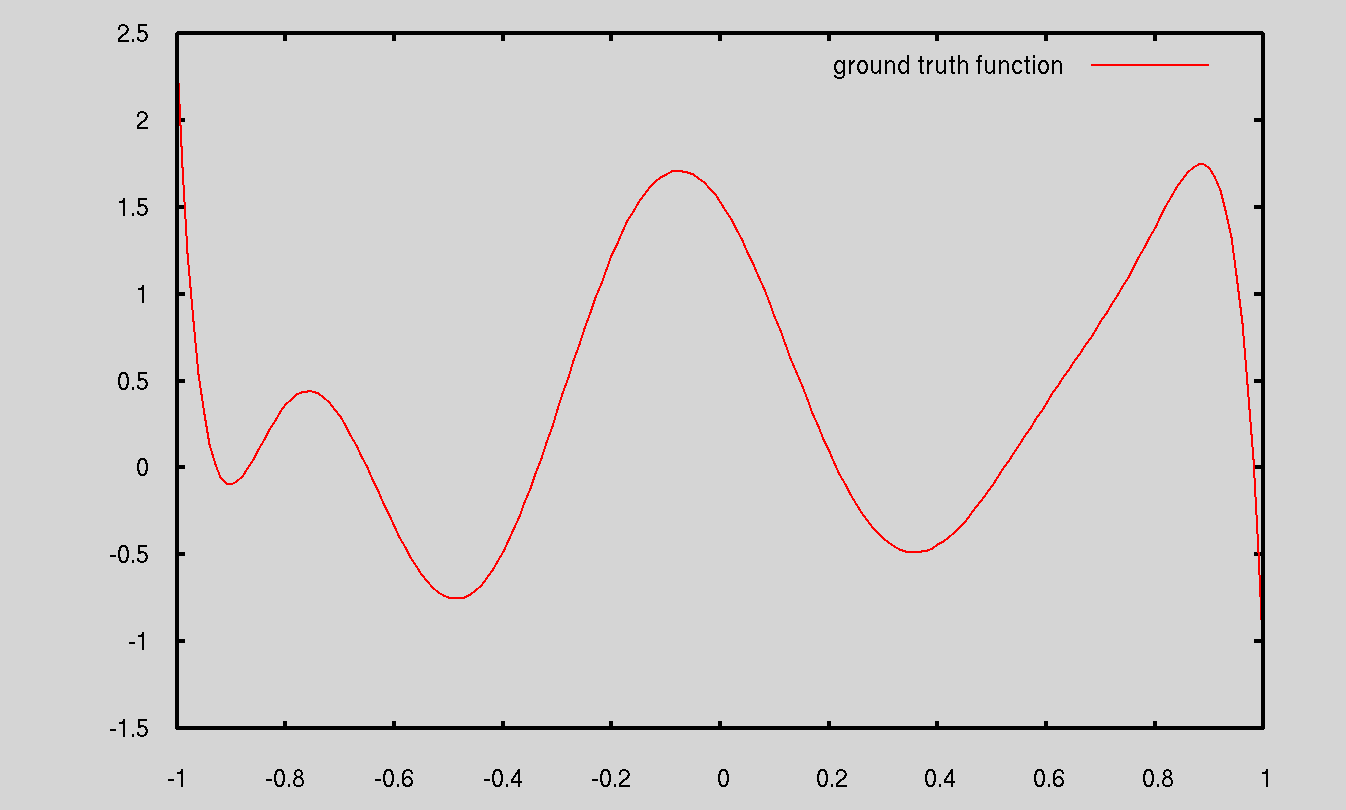}
\caption{A ground truth function $f$ sampled from the model
\eqref{eq:truefunc}.}
\end{center}
\label{fig:truefuncsample}
\end{figure}

Equilibration was run for $50$ steps, as described in
Section~\ref{sec:simucontrol}. We presented a random sequence of $1,000$
samples $(z,f(z))$ with uniform random $z\in [-1;1]$.

We tested values $\eps,\beta\in \{0.5,0.1,0.01\}$, thus including
relatively large and small values. We tested \algoname,
Pessimistic \algoname, and Centered
\algoname. For reference we also compare to ordinary SGD with learning
rate $\eps\beta$, according to Theorem~\ref{thm:sgd}. The results are
reported in Fig.~\ref{fig:linregresults}.

For $\beta=0.01$, the curves are virtually indistinguishable. For
$\beta=0.1$ there are some slight differences: Centered \algoname is
virtually indistinguishable from SGD, in accordance with theory, while
Pessimistic \algoname seems to have a slightly lower error, and \algoname
a slightly larger one.

Results are more interesting with the more aggressive setting $\beta=.5$.
Here, once more, Centered \algoname stays very close to SGD, but the
differences get more pronounced for the other variants. For small $\eps$,
Pessimistic \algoname has the best performance while \algoname is worse.
However, when $\eps$ gets larger, Pessimistic \algoname becomes less
numerically stable.

The most surprising result is with the most aggressive setting
$\beta=.5$, $\eps=.5$, corresponding to the largest learning rate
$\eps\beta$. With this setting, SGD gets unstable (the learning rate
is too large), and so do Pessimistic \algoname and Centered \algoname.
However, \algoname optimizes well. So, in this experiment, \algoname
seems to be more stable than SGD and supports larger learning rates, with
settings for $\beta$ and $\eps$ that clearly do not have to be very
small.

\begin{figure}
\begin{center}
\includegraphics[width=.3\textwidth]{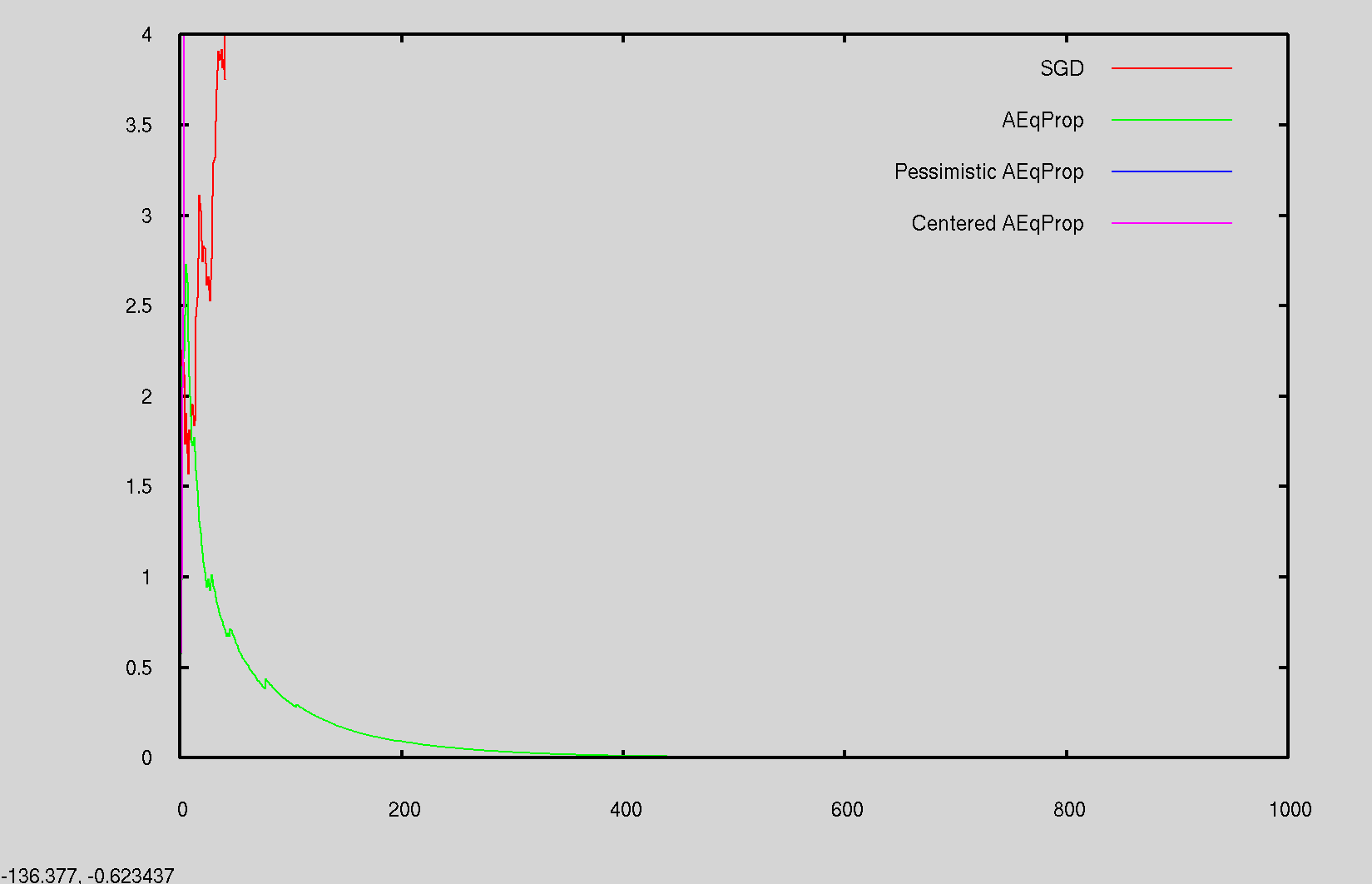}
\includegraphics[width=.3\textwidth]{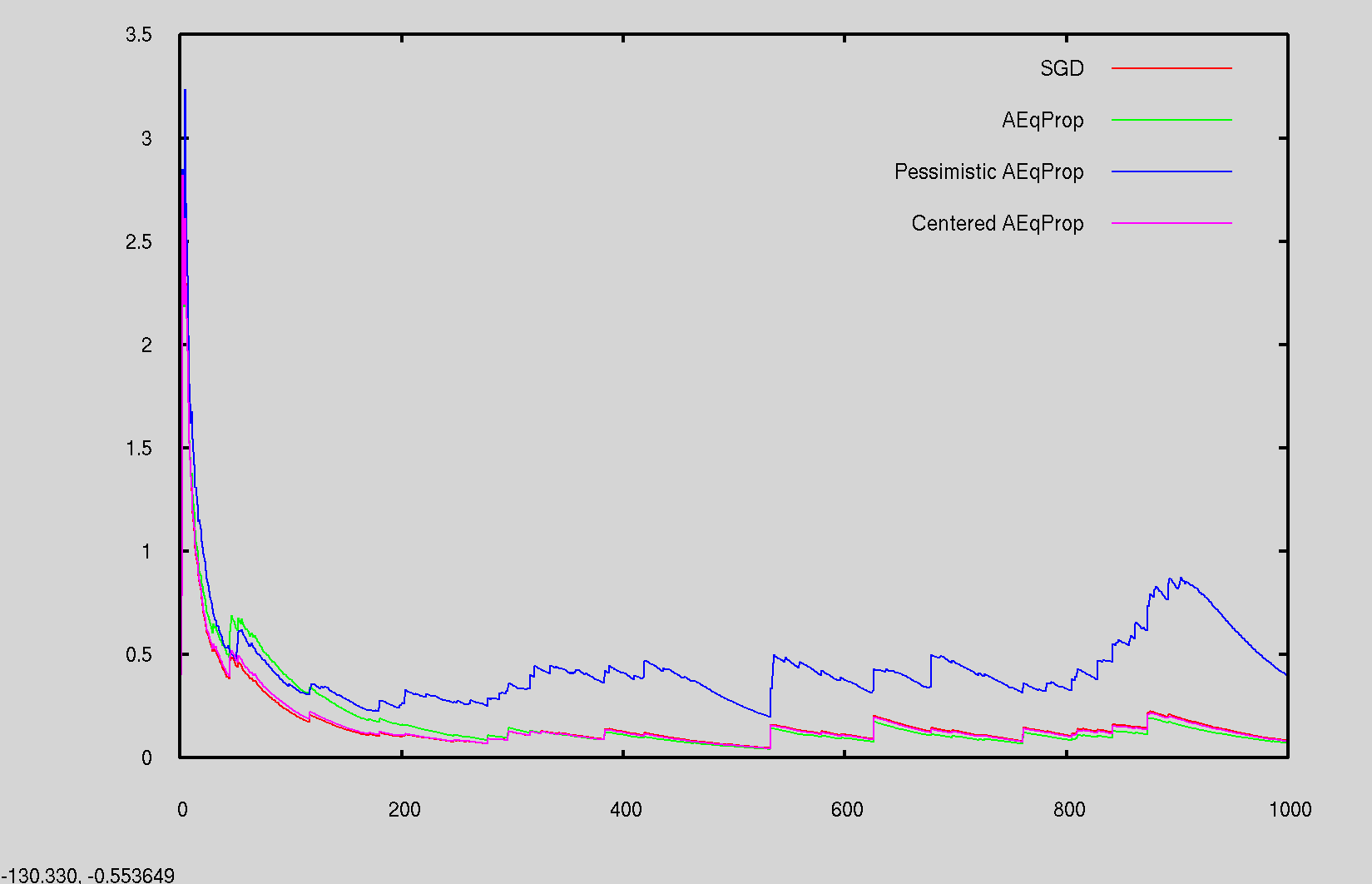}
\includegraphics[width=.3\textwidth]{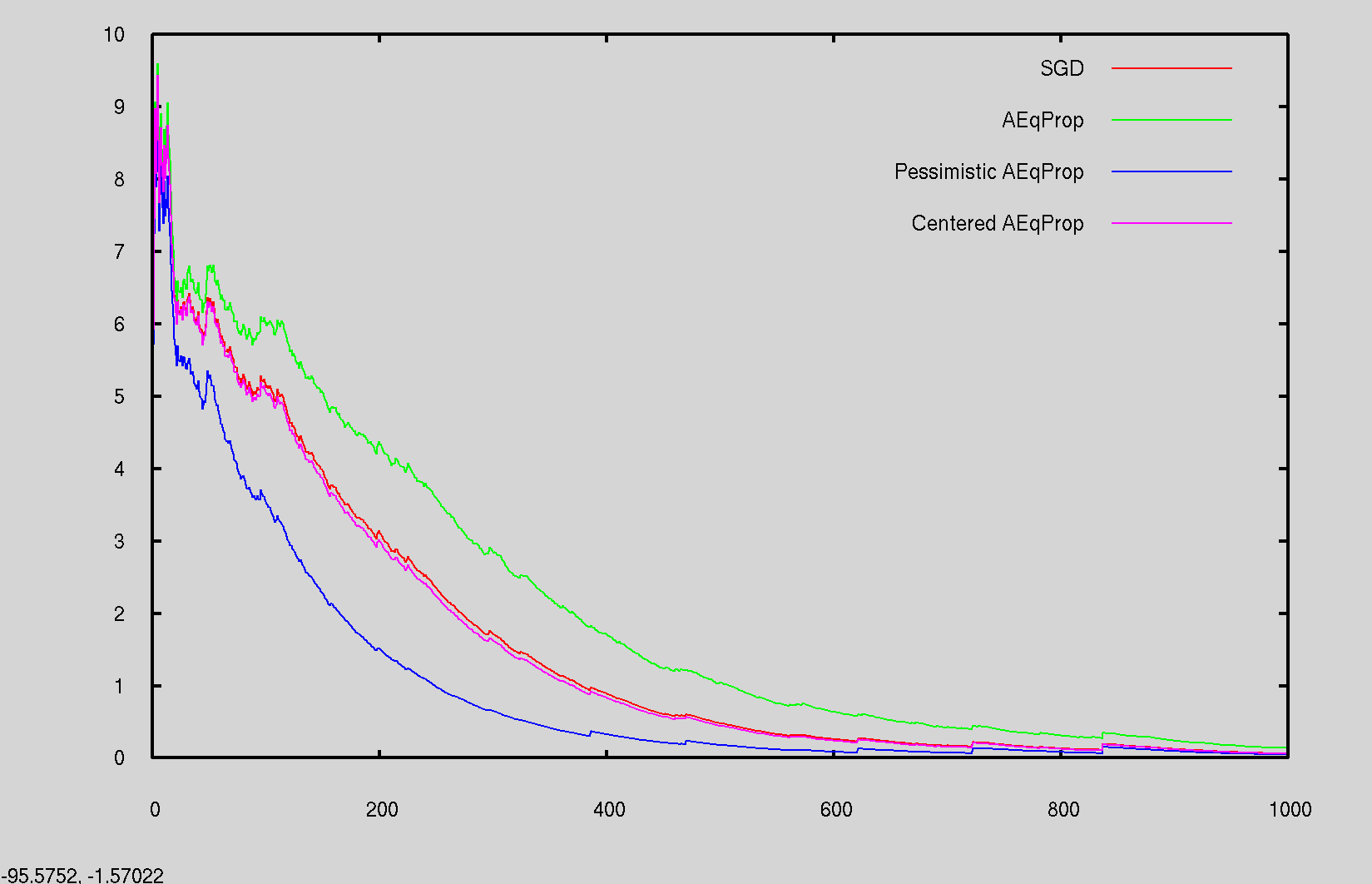}
\\
\includegraphics[width=.3\textwidth]{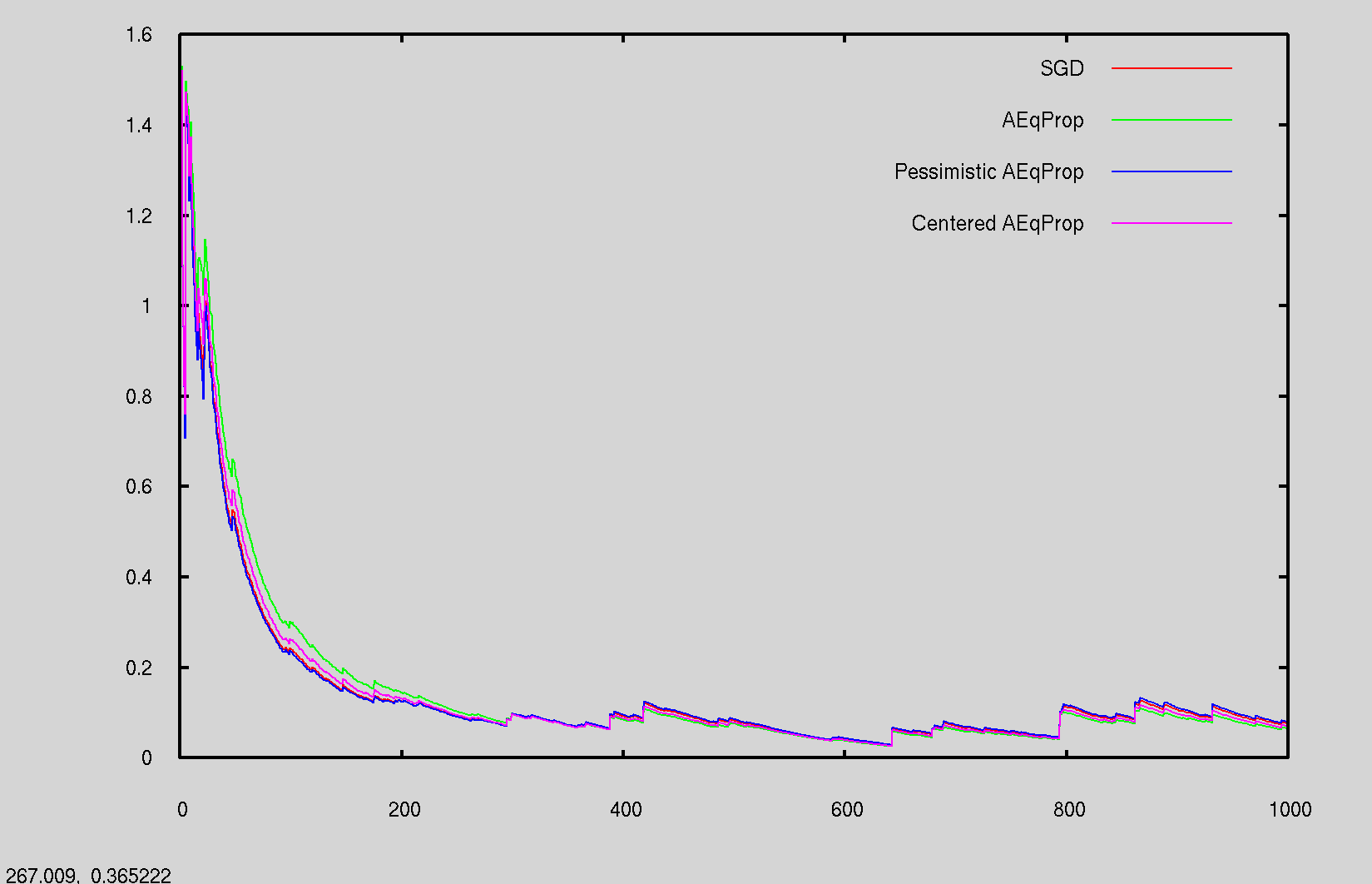}
\includegraphics[width=.3\textwidth]{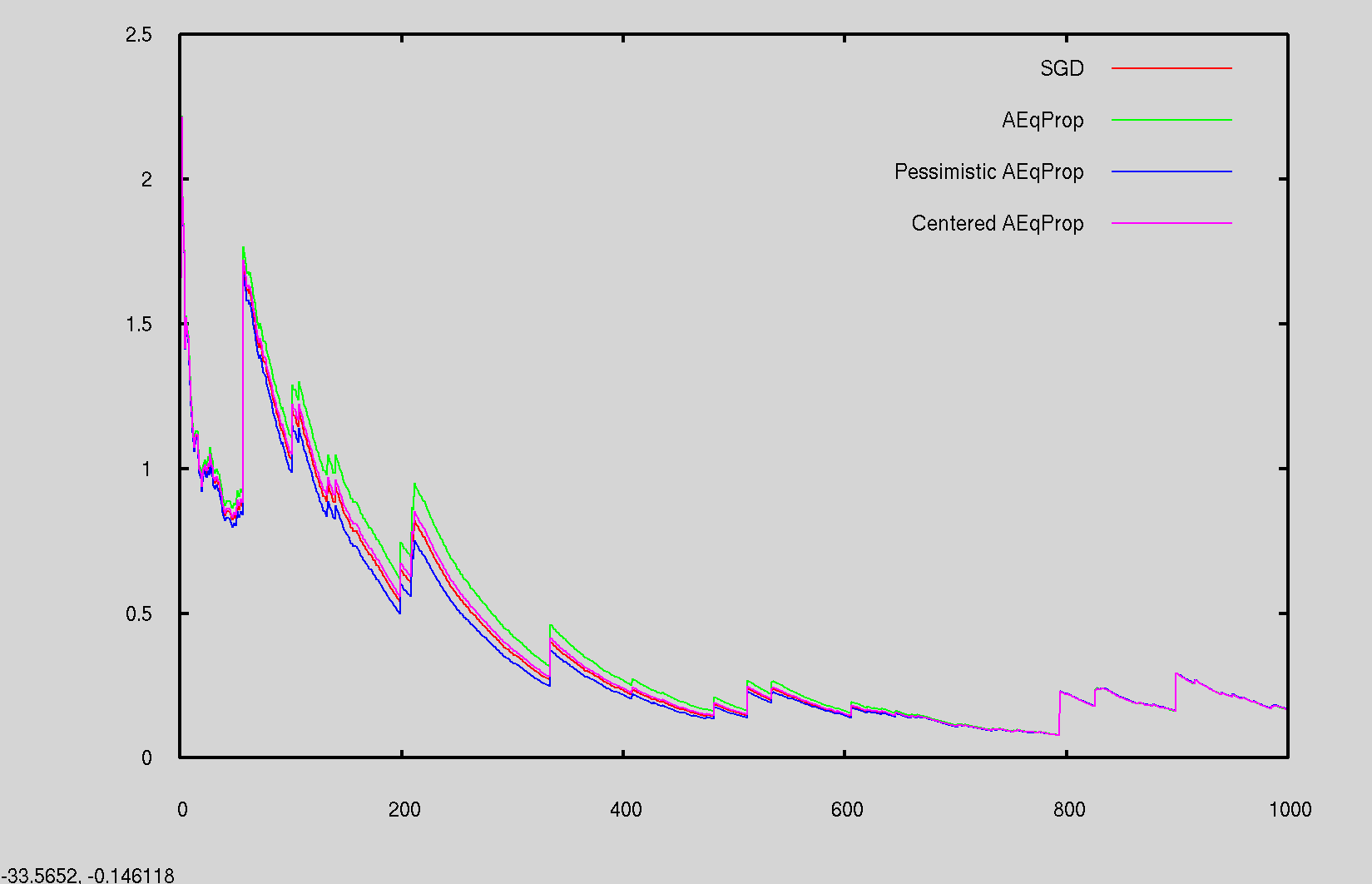}
\includegraphics[width=.3\textwidth]{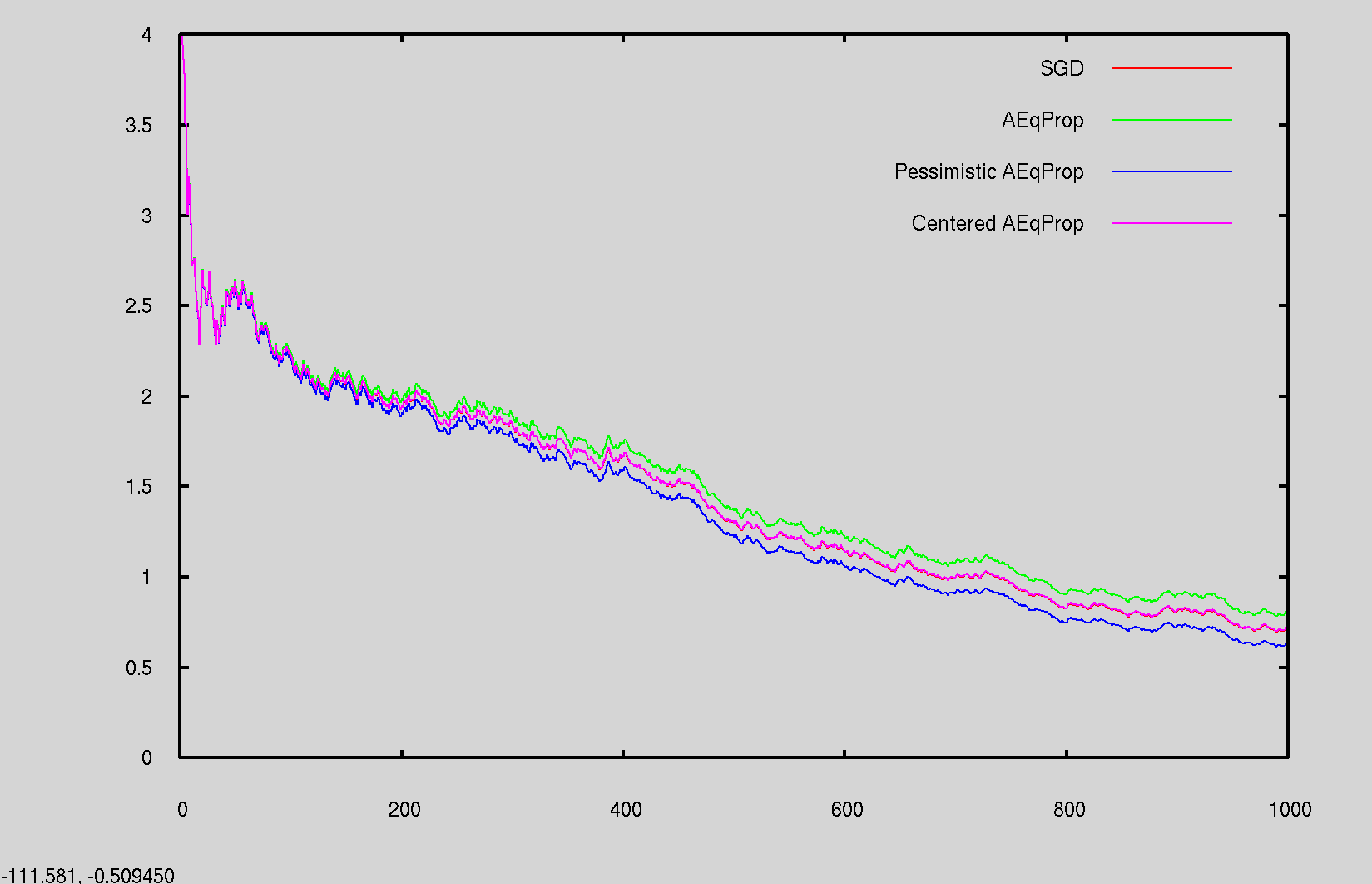}
\\
\includegraphics[width=.3\textwidth]{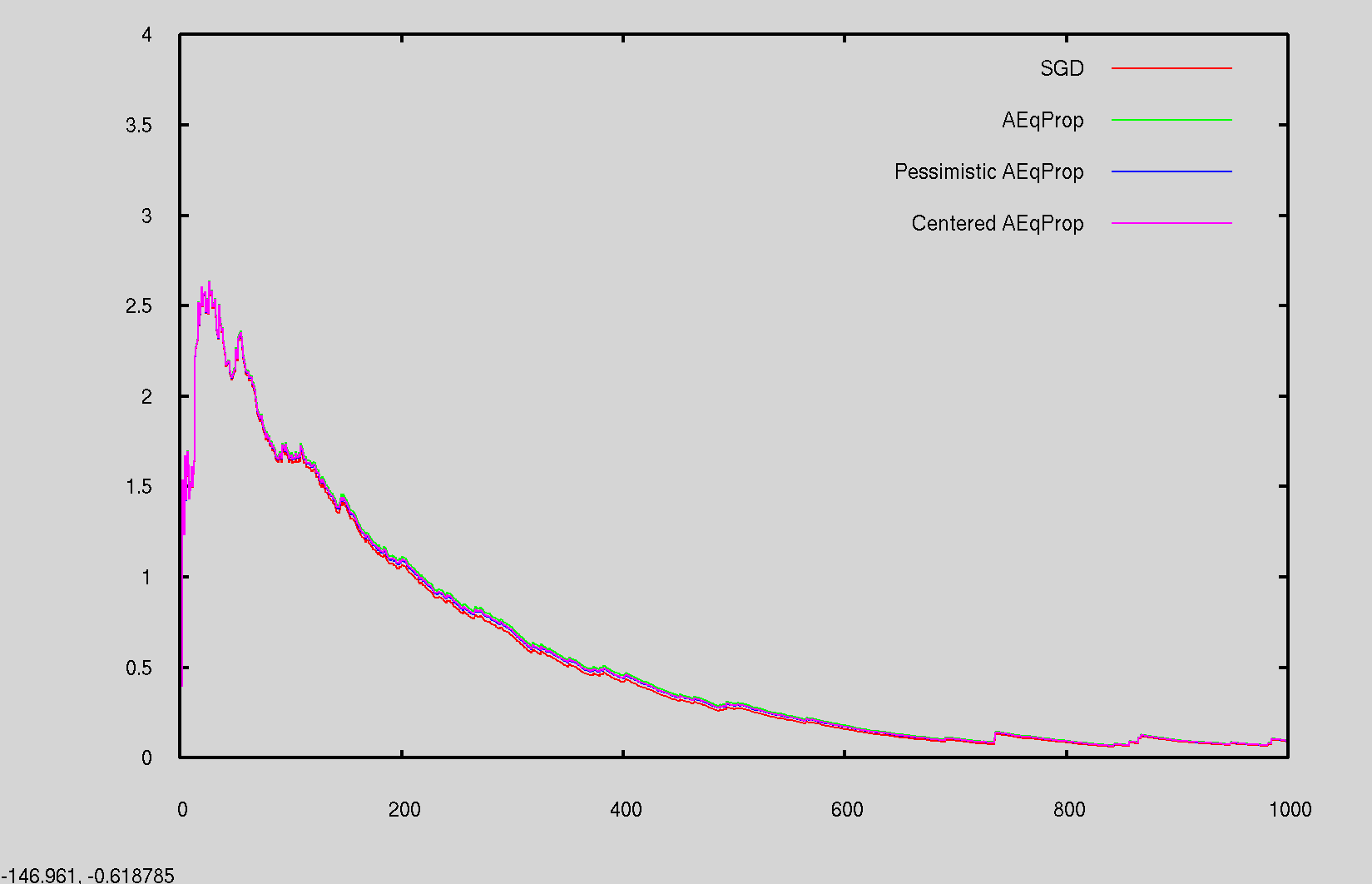}
\includegraphics[width=.3\textwidth]{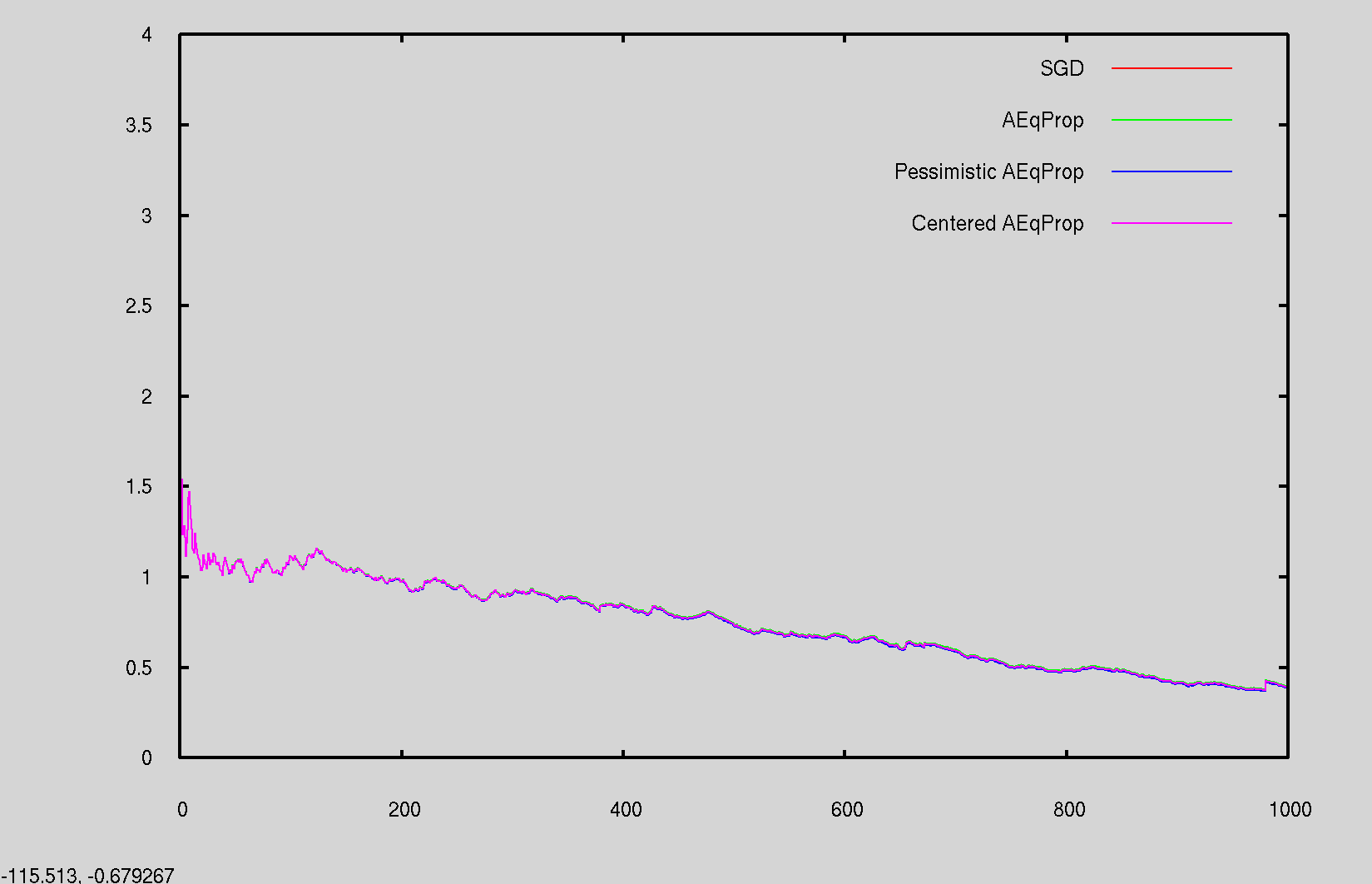}
\includegraphics[width=.3\textwidth]{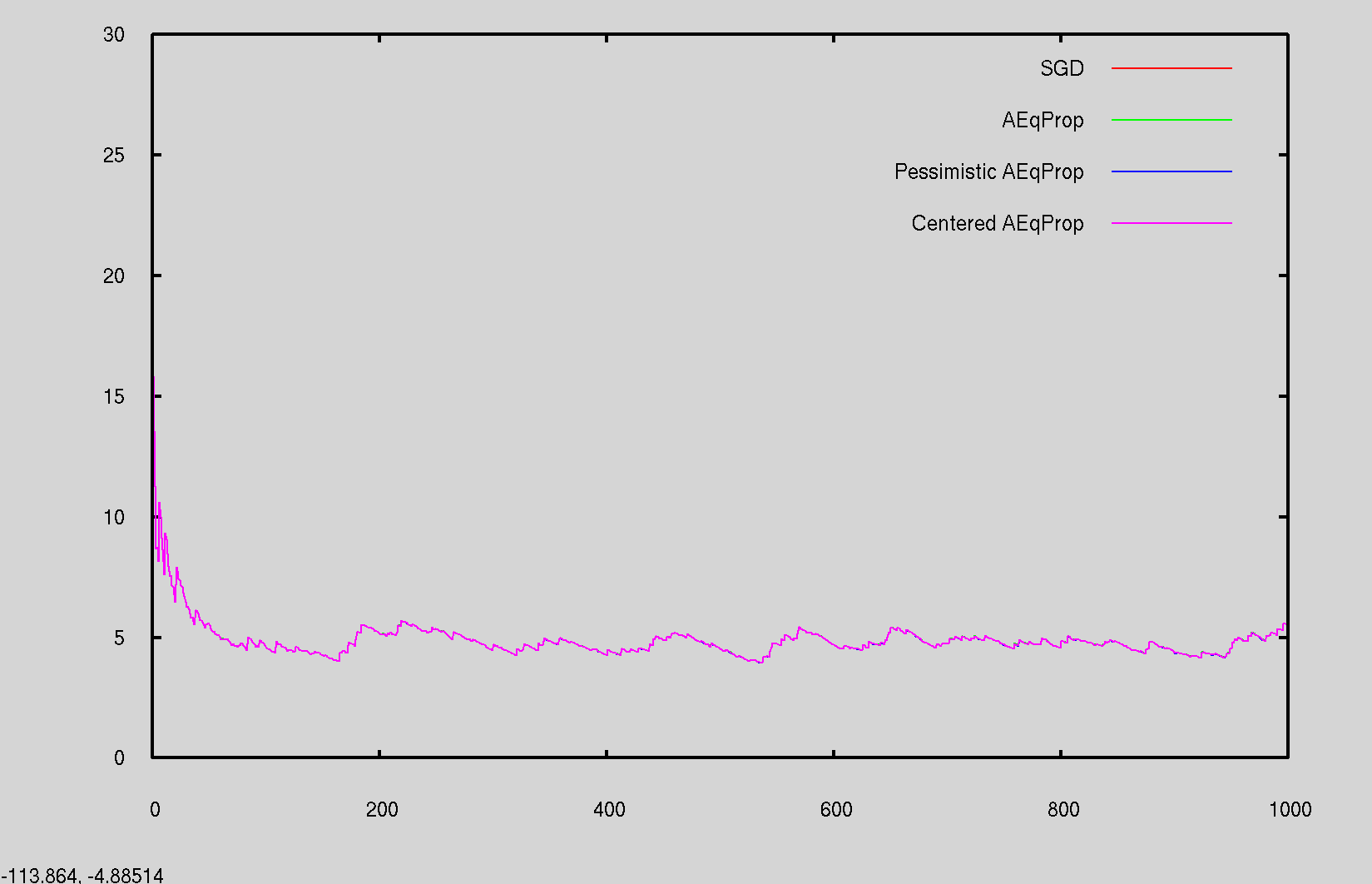}
\end{center}
\caption{SGD, \algoname, Pessimistic \algoname, and Centered
\algoname on the linear regression problem for various values of $\beta$
and $\eps$. Top row: $\beta=0.5$. Middle row: $\beta=0.1$. Bottom row:
$\beta=0.01$. Left column: $\eps=0.5$. Middle column: $\eps=0.1$. Right
column: $\eps=0.01$.}
\label{fig:linregresults}
\end{figure}

The numerical instability of Pessimistic and Centered \algoname for large
$\beta$ is
due to the energy $-\beta C(s,y)$: since $C$ can tend to $\infty$, this
energy is minimized when $C$ is infinite, with $s\to\infty$. This can be
corrected simply by adding $\norm{s}^2$ to the energy function $E$ of the
system: then as long as $\beta<2$ the energy is bounded below and cannot
diverge. This changes the prediction model, however, inducing a
preference for smaller values of $s$.

This is tested in Fig.~\ref{fig:linregresults_regul}:
with $\beta=1.5$, Pessimistic and Centered
\algoname are stable again, and all variants of \algoname seem to work
even in a regime where SGD itself is unstable
(Fig.~\ref{fig:linregresults_regul}, left and middle).
However, convergence seems to be slower (we used $5,000$ samples instead of
$1,000$ in the figure).
Once stabilized, it seems again that Pessimistic \algoname tends
to have smaller error than the other two variants.

\begin{figure}
\begin{center}
\includegraphics[width=.3\textwidth]{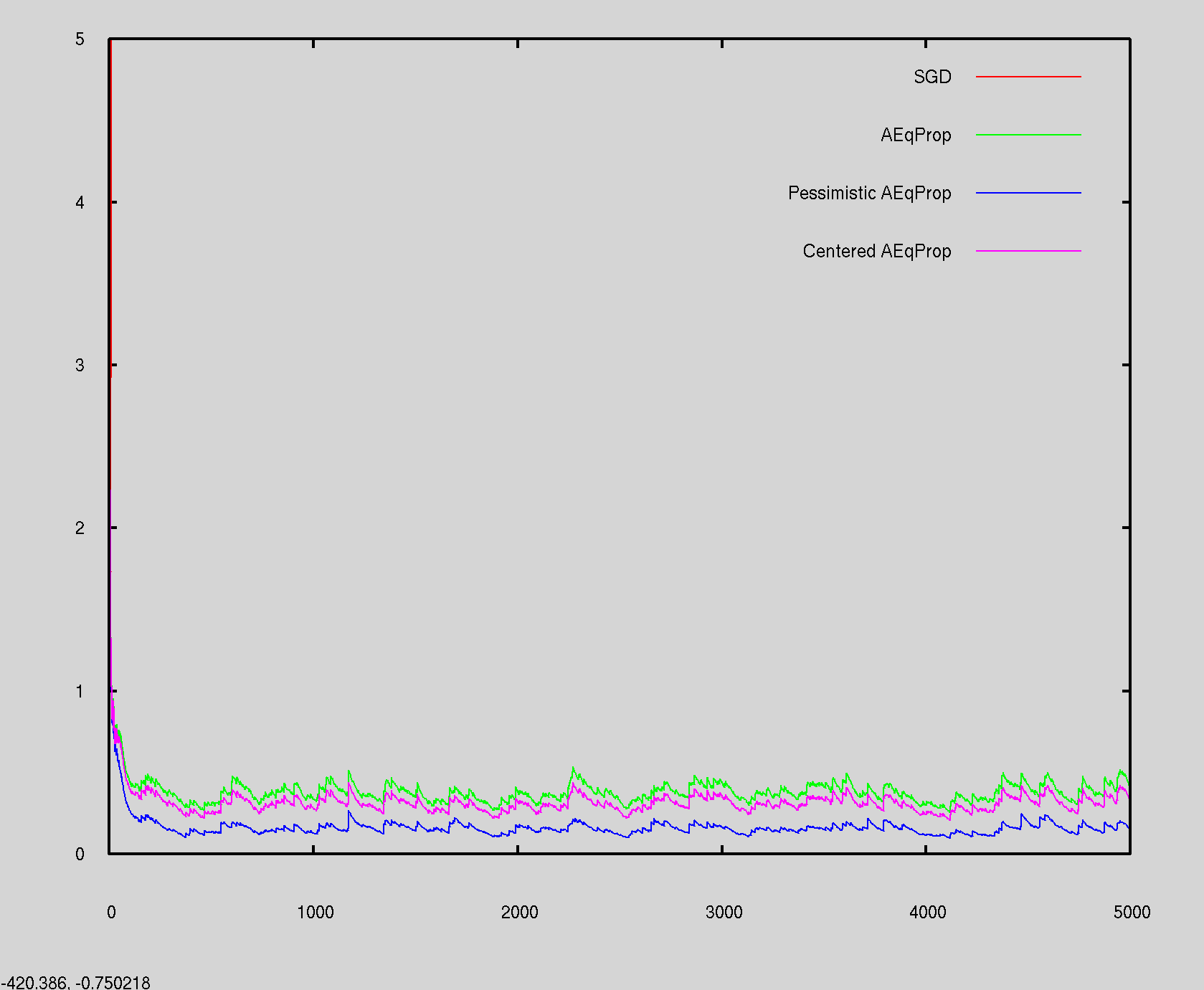}
\includegraphics[width=.3\textwidth]{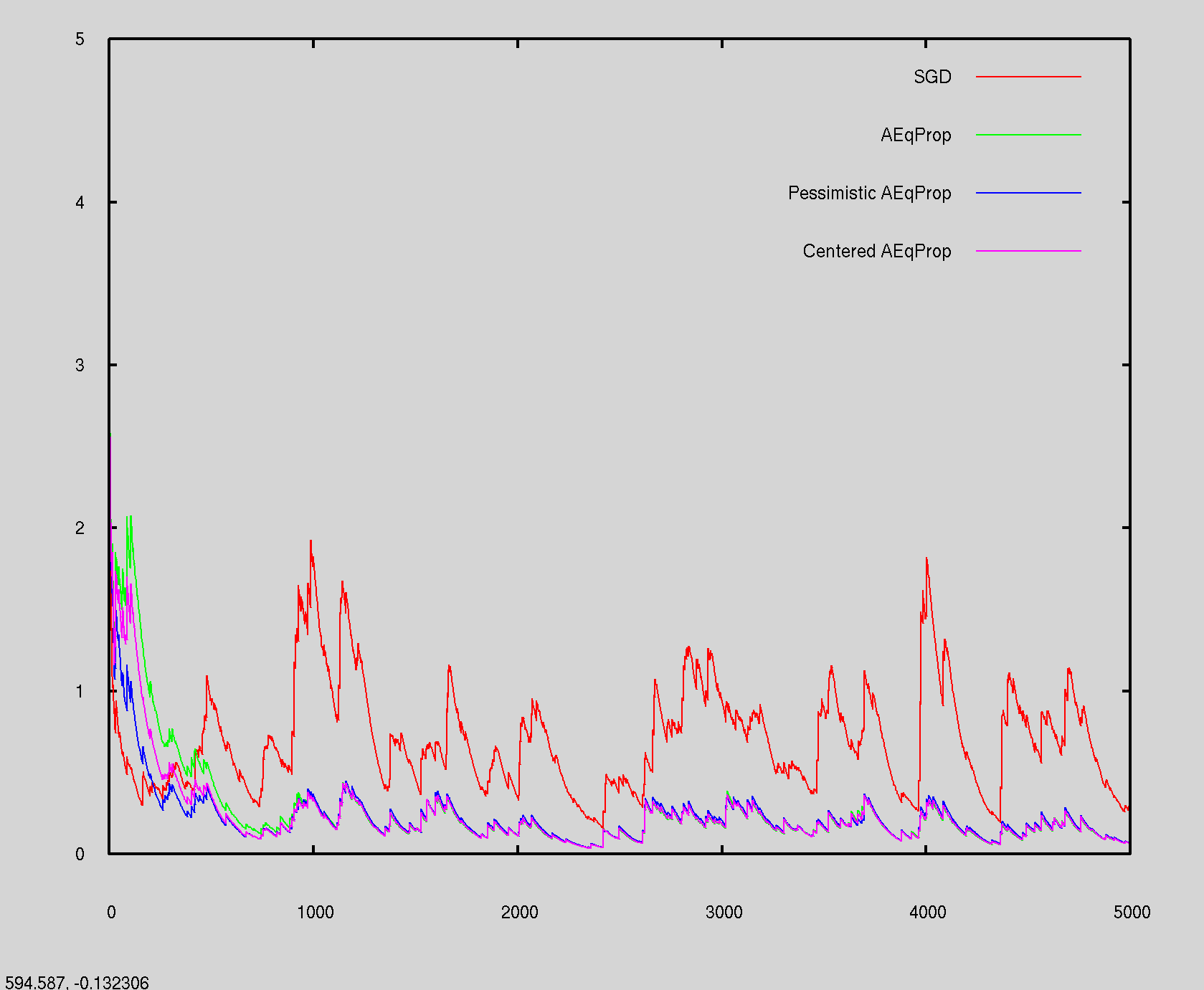}
\includegraphics[width=.3\textwidth]{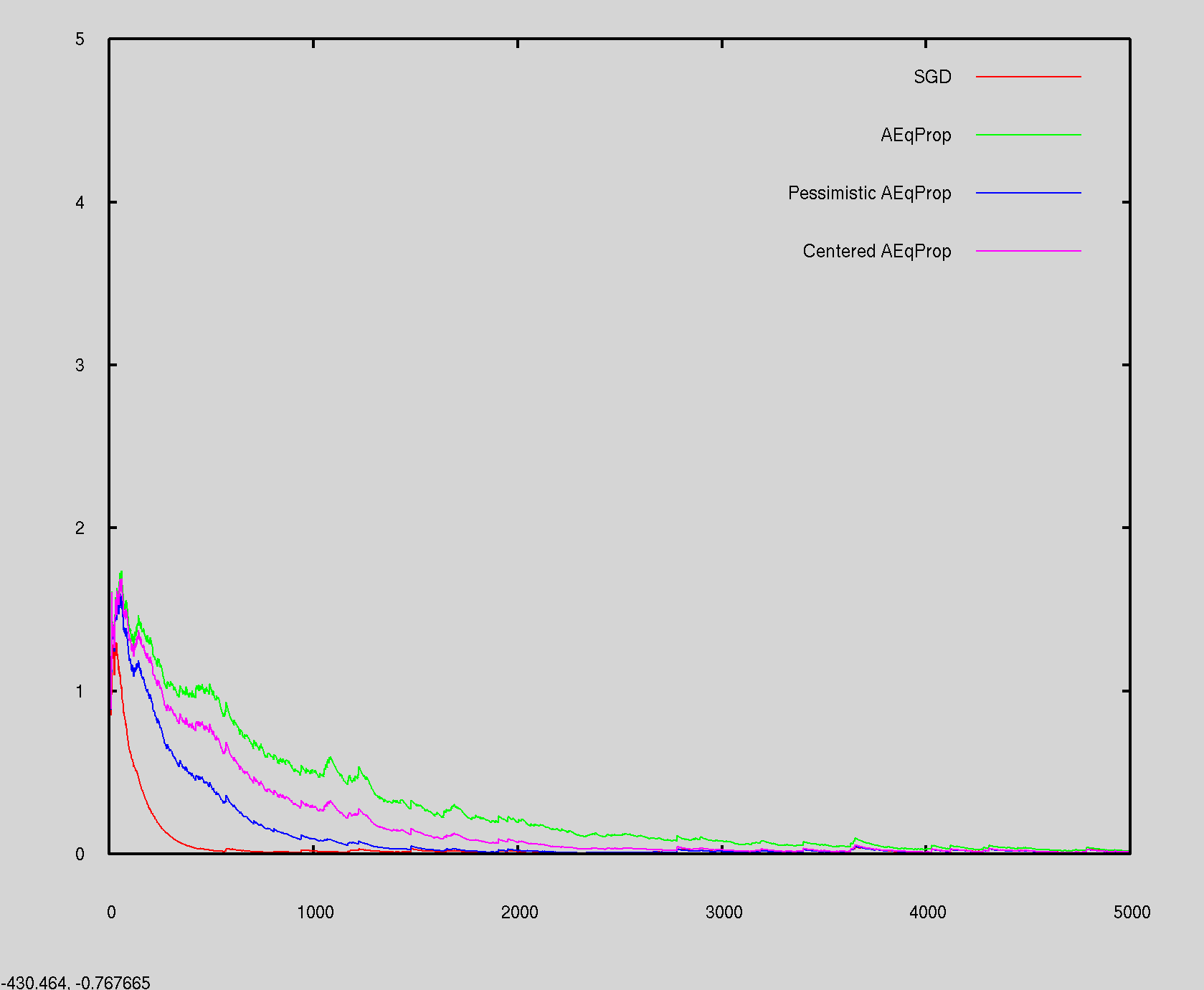}
\end{center}
\caption{SGD, \algoname, Pessimistic \algoname, and Centered
\algoname on the linear regression problem for $\beta=1.5$ and
for $\eps=0.5$ (left)
$\eps=0.1$ (middle), and $\eps=0.01$ (right), with an added term
$\norm{s}^2$ to the energy function.}
\label{fig:linregresults_regul}
\end{figure}

\todo{put some noise on $y$? there's already the noise from choosing the
sample $x$}

\subsection{Hopfield-Like Networks and Real Datasets}
\label{sec:hopfield}

Following the simulations of
\cite{scellier2017equilibrium,ernoult2019updates,laborieux2021scaling}
with Eqprop, we test \algoname on dense and convolutional Hopfield-like
networks. We train the networks on MNIST and FashionMNIST.% and CIFAR-10.
%Although Hopfield networks are abstract models (not `physical systems'), one may think of them as idealized models of physical systems.

\paragraph{Dense Hopfield-like network.}
We consider the setting of classification.
In a layered Hopfield network, the state variable is of the form $s = (s_1, s_2, \ldots, s_N)$ where $s_1, s_2, \ldots, s_{N-1}$ are the `hidden layers' and $s_N$ is the `output layer'. Denoting $s_0 = x$ the inputs, the Hopfield energy function is
\begin{equation}
    E(\theta,x,s) = \sum_{k=1}^N \frac{1}{2} \| s_k \|^2 + \sum_{k=1}^N E_k^{\rm dense}(w_k, s_{k-1}, s_k) - \sum_{k=1}^N b_k^\top s_k,
\end{equation}
where
\begin{equation}
    E_k^{\rm dense}(w_k, s_{k-1}, s_k) := - s_k^\top w_k s_{k-1}
\end{equation}
is the energy of a dense interaction between layers $k-1$ and $k$, parameterized by the $\dim(s_{k-1}) \times \dim(s_k)$ matrix $w_k$. The set of parameters of the model is $\theta = \{ w_k, b_k \mid 1 \leq k \leq N \}$, where $w_k$ are the weights and $b_k$ the biases. We recall that the state of the model at equilibrium given an input $x$ is $s(\theta,x) = \underset{s \in \mathcal{S}}{\min} \; E(\theta,x,s)$, where $\mathcal{S}$ is the state space, i.e. the space of the state variables $s$. We choose $\mathcal{S}$ of the form $\mathcal{S} = \prod_{k=1}^N [p_k,q_k]^{\dim(s_k)}$, where $[p_k,q_k]$ is a closed interval of $\mathbb{R}$ and $\dim(s_k)$ is the number of units in layer $k$. This choice of $\mathcal{S}$ ensures that $s(\theta,x)$ is well defined (there exists a minimum for $E$ in $\mathcal{S}$) and also introduces nonlinearities: for fixed $\theta$, $s(\theta,x)$ is a nonlinear response of $x$.

\paragraph{Convolutional Hopfield-like interactions.}
Convolutional layers can be incorporated to the network by replacing some of the dense interactions $E_k^{\rm dense}$ in the energy function by convolutional interactions:
\begin{equation}
    E_k^{\rm conv}(w_k, s_{k-1}, s_k) := - s_k\bullet\mathcal{P}\left(w_k\star s_{k-1}\right).
\end{equation}
In this expression, $w_k$ is the kernel (the weights), $\star$ is the convolution operation, $\mathcal{P}$ is the average pooling operation, and $\bullet$ is the scalar product for pairs of tensors with same dimension.
%In particular, $s_k$ and $\mathcal{P}\left(w_k\star s_{k-1}\right)$ are tensors with same size.
%Below we explain why we choose the average pooling operation rather than the max pooling operation.

\paragraph{Cost function.} We use the squared error cost function $C(s,y)
= \| s_N - y \|^2$, where $s_N$ is the output layer and $y$ is the
one-hot code of the label (in the classification tasks studied here).
\TODO{For the feedforward models we should stick to log-loss}

\paragraph{Energy minimization.} For our simulations, we require a numerical method to minimize the global energy with respect to the `floating variables' ($s_k$, $w_k$ and $b_k$). For each variable $z \in \{ s_k,w_k,b_k \mid 1 \leq k \leq N \}$, we note that the global energy $\mathcal{E}$ is a quadratic function of $z$ given the state of other variables fixed. That is, the global energy as a function of $z$ is of the form $\mathcal{E}(z) = a z^2 + b z + c$, for some real-valued coefficients $a$, $b$ and $c$. The minimum of $\mathcal{E}(z)$ in $\mathbb{R}$ is obtained at $z = - b / 2 a$, and therefore, the minimum in the interval $[p_k,q_k]$ is obtained at $z = \min(\max(p_k, - b / 2 a), q_k)$. We use this property to optimize $\mathcal{E}$ with the following strategy: at each step, we pick a variable $z$ and compute the state of $z$ that minimizes $\mathcal{E}$ given the state of other variables fixed. Then we pick another variable and we repeat. We repeat this procedure until convergence.
%We explain in details the convergence criterion used in Appendix \ref{sec:simulation-details}.
%We note that this strategy does not apply to max pooling operations.

\paragraph{Homeostatic control.} To accelerate simulations, we use the following method to save computations in the first phase of training (homeostatic phase): first, we keep $\theta$ fixed to its current value and we calculate the state of the layers ($s_1, s_2, \ldots, s_N$) that minimizes the energy $E + \beta C$ for fixed $\theta$; then we calculate the value of the control knob $u$ for which the parameter $\theta$ is at equilibrium. Recalling that $\mathcal{E} = ||u-\theta||^2 / 2\epsilon + E + \beta C$, this value of $u$ can be computed in one step as it is characterized by $\frac{\partial \mathcal{E}}{\partial \theta} = 0$, i.e. $u = \theta + \epsilon \frac{\partial E}{\partial \theta}$, where $\frac{\partial E}{\partial \theta}$ is the \textit{partial} derivative of $E$ wrt $\theta$ (not the \textit{total} derivative).
For instance, we have $\frac{\partial E}{\partial b_k} = - s_k$ for the bias $b_k$, and $\frac{\partial E}{\partial w_k} = \frac{\partial E_k^{\rm dense}}{\partial w_k} = - s_{k-1} s_k^\top$ for dense weight $w_k$.

\paragraph{MNIST and FashionMNIST.} We train a dense Hopfield-like network and a convolutional Hopfield-like network on MNIST and FashionMNIST. Both networks have an input layer of size $1\times28\times28$ and an output layer with $10$ units. In addition, the dense network has one hidden layer of $2048$ units, whereas the convolutional network has two hidden layers of size a $32\times12\times12$ and $64\times4\times4$ : the first two interactions ($x-s_1$ and $s_1-s_2$) are convolutional with kernel size $5\times5$, zero padding, and average pooling; the last interaction ($s_2-s_3$) is dense.

\paragraph{Baseline.} Since Hopfield networks are defined by energy
minimization, it is not possible to apply backpropagation directly as a
baseline. Still, it is possible to compute gradients via the whole
procedure used to find the approximate energy minimizer: unfold the whole graph
of computations during the free phase minimization (with $\beta=0$), and
compute the gradient of the final loss with respect to the parameters.
Then,
take one step of gradient descent for each parameter $\theta_k$,
with step size $\beta \eps_k$. This is the baseline denoted as
\emph{autodiff} in the table.

However, the autodiff procedure seems to be numerically unstable, for
reasons we have not identified.
%Perhaps propagating gradients through many iterations of the energy minimization step results in exploding and vanishing gradients as in poorly conditioned recurrent network models. \todo{keep?  }

The results are reported in Table~\ref{table:mnist}, and show that
\algoname successfully manages to learn on these datasets, even with
the relatively large $\beta$ used ($0.5$ for dense networks and $0.2$ for
convolution-like networks). Centered
\algoname seems to offer the best precision. \todo{cite
a paper for results with similarly-sized models?}

%\paragraph{CIFAR-10.} We also train a dense network and a convolutional network on CIFAR-10. Both networks have an input layer of size $3\times32\times32$ and an output layer with $10$ units. The dense network has one hidden layer of $2048$ units. The convolutional network has four hidden layers of size $128\times16\times16$, $256\times8\times8$, $512\times4\times4$, $512\times2\times2$ : the first four interactions are convolutional with kernel size $3\times3$, padding $1$, and average pooling ; the last interaction is dense.

\begin{table}[ht!]
  \caption{Simulation results on MNIST and FashionMNIST. We train dense networks and convolutional Hopfield-like networks. For each experiment, we perform five runs of 200 epochs. For each run we compute the mean test error rate and the mean train error rate over the last 50 epochs. We then report the mean and standard deviation over the 5 runs. Implementation details are provided in Appendix \ref{sec:simulation-details}.
  }
  \label{table:mnist}
  \centering
  \begin{tabular}{ccccc}
    \toprule
    \multicolumn{2}{c}{} & & \multicolumn{2}{c}{Error (\%)}                   \\
    \cmidrule(r){4-5}
    Task                          & Network & Training Method       & Test & (Train) \\
    \midrule
    \multirow{8}{*}{MNIST}        & \multirow{4}{*}{Dense Hopfield-like} & Optimistic \algoname         & $2.36 \pm 0.07$ & ($0.10$) \\
                                  &                                & Pessimistic \algoname & $1.38 \pm 0.03$ & ($0.09$) \\
                                  &                                & Centered \algoname & $\bf{1.29 \pm 0.04}$ & ($0.00$) \\
                                  &                                & Autodiff & $72.37 \pm 35.35$ & ($71.96$) \\
    \cmidrule(r){2-5}
                                  & \multirow{4}{*}{Convolutional Hopfield-like} & Optimistic \algoname & $1.12 \pm 0.07$ & ($3.62$) \\
                                  &                                & Pessimistic \algoname & $1.11 \pm 0.08$ & ($1.73$) \\
                                  &                                & Centered \algoname & $\bf{0.76 \pm 0.05}$ & ($0.24$) \\
                                  &                                & Autodiff & $89.63 \pm 0.96$ & ($89.61$) \\
    \midrule
    \multirow{8}{*}{FashionMNIST} & \multirow{4}{*}{Dense Hopfield-like} & Optimistic \algoname & $10.53 \pm 0.12$ & ($2.30$) \\
                                  &                                & Pessimistic \algoname & $10.73 \pm 0.07$ & ($7.46$) \\
                                  &                                & Centered \algoname & $\bf{9.28 \pm 0.10}$ & ($3.69$) \\
                                  &                                & Autodiff & $10.18 \pm 0.32$ & ($4.25$) \\
    \cmidrule(r){2-5}
                                  & \multirow{4}{*}{Convolutional Hopfield-like} & Optimistic \algoname & $10.69 \pm 0.17$ & ($15.11$) \\
                                  &                                & Pessimistic \algoname & $11.16 \pm 0.20$ & ($9.89$) \\
                                  &                                & Centered \algoname & $\bf{9.17 \pm 0.19}$ & ($7.00$) \\
                                  &                                & Autodiff & $29.55 \pm 30.47$ & ($28.27$) \\
    \bottomrule
  \end{tabular}
\end{table}

%These results can also be compared with prior works on Hopfield networks trained via Eqprop. On MNIST, \cite{ernoult2019updates} achieved 1.95\% test error rate with a dense network, and 1.02\% test error rate with a convolutional network. We speculate that, compared to this earlier work, the better results reported here  are mostly due to better initialization of the weights and are not specific to \algoname. On CIFAR-10, \cite{laborieux2021scaling} achieved 11.8\% test error rate using a convolutional network with \textit{max-pooling}.

%\footnote{The code is available at \url{https://github.com/bscellier/agnostic-equilibrium-propagation}}

\section{Related Work}

There is a considerable amount of literature on the design of fast and
energy-efficient learning systems. While some aim to improve the digital
hardware for running and training existing deep learning algorithms,
others focus on designing novel algorithms for neural network training
and inference on new energy-efficient hardware. We can differentiate this
literature according to the presence or not of an explicit model of the physical
computation performed.

\paragraph{Explicit approaches.} A first approach is to improve the hardware for running neural networks and training them via backpropagation. For instance, \cite{courbariaux2015binaryconnect} %\citep{hubara2016binarized}
explore the use of specialized digital processors for low-precision
tensor multiplications, whereas
\cite{ambrogio2018equivalent,xia2019memristive} investigate the use of
crossbar arrays to perform matrix-vector multiplications in analog.
However, given the mismatches in analog devices, the latter approach
requires mixed digital/analog hardware. These approaches can
be classified as explicit, as the state of the underlying system
can be expressed as $s = f(\theta,x)$ where $f = f_N \circ \ldots \circ
f_2 \circ f_1$ is the composition of (elementary) functions defined by
analytical formulae. %Thus, the details of the actual physical implementation are abstracted away.

\emph{Physics-aware training} \citep{wright2022deep} is a hybrid physical-digital approach in which
inference is carried out on an energy-efficient physical device, but
parameter training is done using gradients computed from an explicit digital model
of the physical device. This is demonstrated on machine learning
tasks using various examples of physical realizations
(optical, mechanical, electronic).

\emph{Spiking neural networks} (SNNs) are networks of individual units that communicate through low-energy electrical pulses, mimicking the spikes of biological neurons \citep{zenke2021visualizing}.
Most SNN models are explicit, and are confronted to the problem of
`differentiation through spikes', which arises when using the chain rule
of differentiation to compute the gradients of the loss\todo{REF}.
However, the implicit Eqprop framework has also been used to train SNNs \citep{mesnard2016towards,o2019training,martin2021eqspike}.

\todo{Perhaps some day, discuss other possible physical realizations? electrical circuits with memristors, ?proteins... But stick to related work, not perspectives}

\paragraph{Implicit approaches.} In contrast,
implicit approaches aim to harness the underlying physical laws of the
device for training and inference. These laws are
seldom in the form of explicit analytical formulae; rather, they are
often
characterized by the minimization
of an \emph{energy} function. For instance, the equilibrium state of an
electrical circuit composed of nonlinear resistors (such as diodes) is
given by the minimization of the so-called \emph{co-content}
\citep{millar1951cxvi}, a nonlinear analogue of electrical power.

Equilibrium propagation (Eqprop) was designed to train neural networks in
this setting \citep{scellier2017equilibrium}. 
First formulated in the context of Hopfield networks, Eqprop has then been deployed in the
context of nonlinear resistive networks and other physical systems
\citep{kendall2020training,scellier2020deep}. The feasibility of Eqprop
training has been further demonstrated empirically on a small resistive
circuit \citep{dillavou2021demonstration}.
A `centered' version of Eqprop was also proposed and tested numerically \citep{laborieux2021scaling}.

\cite{stern2021supervised} propose a variant of Eqprop called
\emph{coupled learning}. In the
second phase, rather than nudging the output unit by adding an energy
term $\beta C$ to the system, the output unit is clamped to $\beta y
+(1-\beta) y_0$, where $y_0$ is the output unit's equilibrium state
without nudging (i.e. the `prediction') and $y$ is the desired output.

Yet, as mentioned in
Section~\ref{sec:introduction} (see also Appendix \ref{sec:eqprop} for more
details), Eqprop as well as the variant of \cite{stern2021supervised} require explicit knowledge of the underlying energy
function, storage of the equilibrium state of the first phase, and
additional mechanisms for updating the parameters. \algoname overcomes all these three
limiting factors. 

The second of these three issues is considered in
another variant of Eqprop, \textit{Continual Eqprop} (CEP)
\citep{ernoult2020equilibrium}: the parameters are updated
continually in the second phase of training to avoid storing the first
equilibrium state. However, the dynamics of the parameters in CEP is
chosen \textit{ad hoc}: no physical mechanism is proposed to account for
the specific dynamics of the parameters.
\cite{anisetti2022learning} propose a different solution to the problem
of storing the first equilibrium state in Eqprop: in the second phase, 
another physical quantity (e.g. the concentration of a chemical) is used to play
the role of error signals.

%In Appendix \ref{sec:eqprop}, we provide more details about Eqprop, the variant of \cite{stern2021supervised}, and the limitations of Eqprop that Agnostic Eqprop solves.

\section{Discussion, Limitations, and Conclusion}
\label{sec:discussion}
In this paper, we have proposed Agnostic equilibrium propagation (\algoname), a novel algorithm by which physical systems can perform stochastic gradient descent without explicitly computing gradients. \algoname leverages energy minimization, homeostatic control and nudging towards the desired output to obtain an accurate estimate of the result of a gradient descent step (Theorem \ref{thm:sgd}). Although it builds upon equilibrium propagation (Eqprop) \citep{scellier2017equilibrium}, \algoname distinguishes itself from Eqprop in the following ways; i) it does not require any explicit knowlegde of the analytical form of the underlying energy function, ii) the equilibrium state at the end of the first phase is not needed to be stored and iii) the parameter update at the end of a gradient step is performed automatically in \algoname and no additional mechanism needs to be introduced to perform this update. Thus, \algoname mitigates major limitations of Eqprop (and its variants) and, in principle, significantly increases the range of hardware on which statistical learning can be performed.  

In addition to showing that \algoname estimates gradient descent steps
accurately, we have also derived a Lyapunov function for \algoname and
proved that this Lyapunov function improves monotonically along the
\algoname trajectory, suggesting enhanced robustness of this algorithm
with non-infinitesimal step sizes, compared to standard SGD. Moreover, we
consider different variants of \algoname (optimistic, pessimistic and
centered), each with desirable properties. In particular, the pessimistic
version of \algoname optimizes an upper bound of the true loss function
whereas the centered version provides a better (second-order)
approximation of the loss. We also illustrated \algoname (and its
variants) numerically with a simple linear regression example as well as
with Hopfield-like networks on the MNIST and
FashionMNIST datasets, showing that \algoname successfully implements gradient
descent learning in practice. 

At this stage, it is germane to examine the main assumptions on which
\algoname rests. Clearly, a large number of physical systems are based on
energy minimization and can be used in the context of \algoname, as they
have already been for Eqprop. Nudging
towards a desired output is a key design principle of \algoname as well
as Eqprop. Such nudging has been realized on model physical systems such as
in \cite{dillavou2021demonstration} and references therein, and
such devices can, in principle, be used in the context of \algoname as
well.
Next, our technical assumptions require the existence and smoothness of
local minimizers of the energy function. Uniqueness of the minimizer is not
required, but then, the system should remain around one of its possible
modes for the duration of training; training will be perturbed if the
system jumps to another minimizer.

Homeostatic control of the parameters is a limiting assumption for \algoname. In
this context, we would like to point that there is quite a bit of
flexibility in terms of the form of the coupling energy $U$. It need not
be quadratic: an arbitrary form of $U$ leads to a Riemannian instead of
Eucliean gradient
descent step, still decreasing the Lyapunov error function. Similarly, the coupling parameter $\eps$ need not be
infinitesimal: relatively large values of $\eps$ (weak coupling) are allowed. In the end, it
is hard to imagine a way to train parameters without some means of
monitoring and controlling them. In \algoname this control only takes the
form of being able to maintain stasis of the current parameters by
adjusting some control knobs.
\algoname offers one way to build gradient
descent from such a generic control mechanism.

Finally, the existence of \algoname has a more general interest, in
showing that generic, physically plausible ingredients such as
homeostatic control and output nudging are enough, in principle, for
natural (bio)physical systems to exhibit genuine
gradient descent learning.

% \todo{}
% Homeostatic control is a strong assumption. The coupling energy $U$ does
% not need to be exactly quadratic thanks to the Riemannian interpretation
% for arbitrary $U$. It has to be relatively strong, but the theorems allow
% for some leeway on $\eps$ and $\beta$ (a large $\eps$ will result in
% per-sample learning rates/Riemannian matrices, large $\beta$ still has
% the Lyapunov function + possibility of symmetric estimator).

%\todo{Discuss jumps in the minimizer?}

\begin{ack}
The authors would like to thank Léon Bottou for insightful comments on
the method and assumptions. The research of BS and SM was partly performed under a project that has received funding from the European Research Council (ERC) under the European Union’s Horizon 2020 research and innovation programme (grant agreement No. 770880).
YB was funded by Samsung for this work.
\end{ack}

\bibliographystyle{abbrvnat}
\bibliography{biblio}

%neuripsonly checklist

\clearpage
\appendix

\section{A Generalization of Theorem~\ref{thm:sgd}: \algoname with
Large $\eps$ or $\beta$, Centered and Pessimistic \algoname}
\label{sec:riemannian-sgd}

We now extend Theorems~\ref{thm:sgd} and~\ref{thm:lyapunov} in the following directions:
\begin{itemize}
\item Variants such as Pessimistic \algoname and Centered \algoname
(Section~\ref{sec:variants}) are
covered.
\item Only one of $\eps$ or $\beta$ needs to tend to $0$.
\item The control energy $U(u,\theta)$ is not necessarily the quadratic
$\norm{u-\theta}^2 / 2$.
\end{itemize}

So at each instant, we set input knobs $x$, control knobs $u$, and possibly (if $\beta>0$) a desired output $y$, and assume that the system reaches an equilibrium $(\theta_\star,s_\star) = \argmin_{(\theta,s)} \; \mathcal{E}(u,\theta,s,x,y,\eps,\beta)$, where
\begin{equation}
\label{eq:generalenergy}
\mathcal{E}(u,\theta,s,x,y,\eps,\beta) :=
U(u,\theta)/\eps+E(\theta,x,s)+\beta C(s,y)
\end{equation}
is the global energy function of the system. (Assumption~\ref{hyp:tech} in Appendix \ref{sec:proofs} ensures the argmin is well-defined.)

Then we follow the \algoname procedure from Section~\ref{sec:aeqprop}.
To cover variants like Pessimistic and Centered \algoname, here we use
two values $\beta_1<\beta_2$ in the two phases of the algorithm.
Namely, we first set a control value $u_{t}$ such
that the equilibrium value $\theta_{t-1}$ does not change when we
introduce the new input $x_t$, the desired output $y_t$ and nudging $\beta_1$. Then we obtain the next parameter by changing the nudging to $\beta_2$ and letting the system reach
equilibrium.

This time, we define the Lyapunov function
\begin{equation}
\label{eq:generallyapunov}
\lyap_{\beta_1;\beta_2}(\theta,x,y)\deq \frac{1}{\beta_2-\beta_1}
\int_{\beta^\prime=\beta_1}^{\beta_2} C(s_{\beta^\prime}(\theta,x,y),y)\d
\beta^\prime
\end{equation}
where as before,
\begin{equation}
s_\beta(\theta,x,y)\deq\argmin_s \{E(\theta,x,s)+\beta C(s,y)\}.
\end{equation}
This Lyapunov function tends to the loss
\begin{equation}
\lyap(\theta,x,y)\deq C(s(\theta,x),y)
\end{equation}
when $\beta_1$ and $\beta_2$ tend to $0$, where $s(\theta,x) \deq\argmin_s E(\theta,x,s) = s_0(\theta,x,y)$ by definition.

The next theorem states that \algoname performs a step of \emph{Riemannian} stochastic gradient descent for the input-output pair $(x_t,y_t)$, with step size (learning rate) $\eps(\beta_2-\beta_1)$, loss function $\lyap_{\beta_1;\beta_2}$, and preconditioning matrix (Riemannian metric) $M$.
When $\beta_1$ and $\beta_2$ both tend to $0$, the Lyapunov function $\lyap_{\beta_1;\beta_2}(\theta_{t-1},x_t,y_t)$ tends to the loss
$\lyap(\theta_{t-1},x_t,y_t)$, thus recovering (Riemannian) stochastic gradient descent with the ordinary loss function. This theorem gives a better description of the behavior of \algoname when $\beta_1$ and $\beta_2$ are not $0$: it still follows the (Riemannian) gradient descent of the closely related function $\lyap_{\beta_1;\beta_2}$.

\begin{thm}
\label{thm:riemannian-sgd}
Let $\theta_{t-1}$ be some parameter value. Let $\beta_1<\beta_2$.
Let $x_t$ and $y_t$ be some
input and output value. Let $u_t$ be a control value such that
\begin{equation}
\label{eq:controlstep2}
\theta_{t-1}=\argmin_\theta \;
\min_s \; \mathcal{E}(u_{t},\theta,s,x_t,y_t,\eps,\beta_1)
\end{equation}
and let
\begin{equation}
\label{eq:nudgestep2}
\theta_{t}=\argmin_{\theta} \;
\min_s \; \mathcal{E}(u_{t},\theta,s,x_t,y_t,\eps,\beta_2)
\end{equation}
working under the technical assumptions of Section~\ref{sec:techhyp}.

Then, 
for any $\eps>0$ and $\beta_2>\beta_1$ (not necessarily tending
to $0$) we have the Lyapunov property
\begin{equation}
\lyap_{\beta_1;\beta_2}(\theta_t,x_t,y_t)\leq
\lyap_{\beta_1;\beta_2}(\theta_{t-1},x_t,y_t).
\end{equation}

Moreover, when
either $\eps$, or $\beta_2-\beta_1$, or both, tend to $0$, we have
\begin{equation}
\label{eq:riemsgdstep}
\theta_t = \theta_{t-1} - \eps(\beta_2-\beta_1) \,
%M_\eps(\theta_{t-1},x_t)^{-1}
M^{-1}
\partial_\theta
\lyap_{\beta_1;\beta_2}(\theta_{t-1},x_t,y_t)+O(\eps^2(\beta_2-\beta_1)^2)
\end{equation}
where $\lyap_{\beta_1;\beta_2}$ is the Lyapunov function
\eqref{eq:generallyapunov}, and where $M$ is the positive definite matrix
\begin{align}
M = M_{\beta_1}^\eps(\theta_{t-1},x_t,y_t)
% M
& \deq \eps \, \partial^2_\theta \left[
\min_s \mathcal{E}(u_t,\theta_{t-1},s,x_t,y_t,\eps,\beta_1) \right] \\
& = \partial^2_\theta \left[ U(u_t,\theta_{t-1}) + \eps \, \min_s \{E(\theta_{t-1},x_t,s)+\beta_1 C(s,y_t)\} \right]
\end{align}
%\begin{equation} 
%M
%\deq \partial^2_\theta
%G_{\beta_1}^\eps(u_t,\theta_{t-1},x_t,y_t), \qquad
%G_\beta^\eps(u,\theta,x,y) \deq
%%U(u,\theta) + \eps \, \min_s \{E(\theta,x,s)+\beta C(s,y)\}.
%\eps \, \min_s \mathcal{E}(u,\theta,s,x,y,\eps,\beta).
%\end{equation}

When $\eps\to 0$ we have
$%\begin{equation}
M=\partial^2_\theta U(u_t^0,\theta_{t-1})+O(\eps)
$ %\end{equation}
where $u_t^0$ is such that $\theta_{t-1}=\argmin_\theta U(u_t^0,\theta)$. In
particular, if $U(u,\theta)=\norm{u-\theta}^2/2$ then $M=\Id+O(\eps)$.

Finally, the Lyapunov function enjoys the following properties. For any $\beta_1<0<\beta_2$, we have
\begin{equation}
\label{eq:lyapbounds}
\lyap_{0;\beta_2}(\theta,x,y)\leq C(s(\theta,x),y)\leq \lyap_{\beta_1;0}(\theta,x,y).
\end{equation}
In particular, Pessimistic \algoname optimizes an upper bound of the loss function.
Moreover, when $\beta_1$ and $\beta_2$ tend to $0$ we have
\begin{equation}
\label{eq:lyaptaylor1}
\lyap_{\beta_1;\beta_2}(\theta,x,y)=C(s(\theta,x),y)+O(\abs{\beta_1}+\abs{\beta_2}),
\end{equation}
and if $\beta_2=-\beta_1=\beta/2$ (Centered \algoname),
\begin{equation}
\label{eq:lyaptaylor2}
\lyap_{-\beta/2;\beta/2}(\theta,x,y)=C(s(\theta,x),y)+O(\beta^2).
\end{equation}
\end{thm}

When $\eps\to 0$, the Riemannian metric $M$ tends to $\partial^2_\theta
U(u_t,\theta_{t-1})$, thus recovering ordinary gradient descent for
quadratic $U$. Note that the Hessian $M$ is always nonnegative definite,
because $\theta_{t-1}$ minimizes the function $\theta \mapsto \min_s \mathcal{E}(u_t,\theta,s,x_t,y_t,\eps,\beta_1)$ by definition \eqref{eq:controlstep2}. Under the technical assumptions below, it is actually positive definite, so that $M^{-1}$ is well-defined.

For fixed, nonzero $\eps$, the metric $M$ depends on $\theta_{t-1}$, on the input $x_t$,
and also on $y_t$ if $\beta_1\neq 0$. Indeed, $E$ depends on $x_t$ in the
definition of $G$, and $u_t$ itself depends on $x_t$ via
\eqref{eq:controlstep2}. This dependency is at first order in $\eps$.

Thus, for large $\eps$, \algoname produces a gradient descent with a
sample-dependent preconditioning matrix. Since this preconditioning may be
correlated with the gradient of the loss for sample $x_t$, this breaks
the property of expected gradients in stochastic gradient descent, and
may introduce bias. This bias disappears when $\eps\to 0$: it is
only a term $O(\eps^2\beta)$ in \eqref{eq:riemsgdstep}.

\clearpage
\section{Proofs}
\label{sec:proofs}

In this section, we prove Theorem \ref{thm:sgd}, Theorem \ref{thm:lyapunov} and Theorem \ref{thm:riemannian-sgd}. We proceed as follows:
\begin{itemize}
    \item In Section \ref{sec:notation}, we introduce the notation.
    \item In Section \ref{sec:techhyp}, we state Definition \ref{def:strictmin} and Assumption~\ref{hyp:tech}, which gather the precise technical assumptions for
    the theorems, such as existence of the minima involved. Proposition \ref{prop:suffhyp}
    gives a simple sufficient condition for these assumptions to hold.
    \item In Section \ref{sec:formula-loss-lyapunov}, we establish
    important formulae relating the loss and Lyapunov function to the
    energy with a free-floating state (Theorem \ref{thm:formula-loss-lyapunov}). We also prove the properties of the Lyapunov function stated in Theorem \ref{thm:riemannian-sgd} (Proposition~\ref{prop:taylor-expansions} and Corollary~\ref{cor:bounds}).
    \item In Section \ref{sec:proof-lyapunov}, we prove the Lyapunov property of Theorem \ref{thm:lyapunov} using Theorem \ref{thm:formula-loss-lyapunov}.
    \item In Section \ref{sec:technical}, we prove a technical lemma (Lemma \ref{lma:technical}), under the assumptions of Assumption~\ref{hyp:tech}.
    \item In Section \ref{sec:proof-riemannian-sgd}, using Lemma \ref{lma:technical} and Theorem \ref{thm:formula-loss-lyapunov}, we prove the Riemannian SGD property of Theorem \ref{thm:riemannian-sgd}.
    \item In Section \ref{sec:proof-sgd}, we prove Theorem~\ref{thm:sgd} (the SGD property) as a corollary of Theorem \ref{thm:riemannian-sgd}, using Theorem \ref{thm:formula-loss-lyapunov} again.
\end{itemize}

%The result is first stated semi-informally as
%Theorem~\ref{thm:riemannian-sgd}, then formally as
%Theorem~\ref{thm:main-taylor-expansion}.

\subsection{Notation}
\label{sec:notation}

Theorems~\ref{thm:sgd}
and~\ref{thm:lyapunov} are particular cases of
Theorem~\ref{thm:riemannian-sgd}.

Since Theorem~\ref{thm:riemannian-sgd}
deals with a fixed input-output pair $(x_t,y_t)$, in
all proofs we assume that $x_t$ and $y_t$ are fixed, and omit them from
the notation all along.

We denote
\begin{equation}
    s_\beta(\theta) \deq \underset{s}{\arg \min} \{ E(\theta,s) + \beta \, C(s) \}
\end{equation}
the equilibrium state with nudging $\beta$ and
\begin{align}
    F(\beta,\theta) \deq & \; \underset{s}{\min} \; \{ E(\theta,s) + \beta \, C(s) \}
    \\ = & \; E(\theta,s_\beta(\theta))+\beta C(s_\beta(\theta))
\end{align}
the minimal energy when $s$ is floating. The loss to optimize is
\begin{equation}
    \mathcal{L}(\theta) \deq C(s_0(\theta)),
\end{equation}
where $s_0(\theta)$ is the equilibrium state without nudging. For every $\beta_1 < \beta_2$, the Lyapunov function is
\begin{equation}
 \mathcal{L}_{\beta_1;\beta_2} \left( \theta \right) \deq \frac{1}{\beta_2-\beta_1} \int_{\beta_1}^{\beta_2} C(s_\beta(\theta)) \, d\beta.
\end{equation}

We then introduce a control variable $u$ and we further augment the
energy of the system by adding a coupling energy $U(u,\theta) / \eps$
between $u$ and $\theta$, scaled by a positive scalar $\eps$. We denote
$G_\beta^\eps(u,\theta)$ the global energy \eqref{eq:generalenergy}
minimized by the system (when $s$ is floating), rescaled by $\eps$:
\begin{align}
\label{eq:def-G}
G_\beta^\eps(u,\theta) &\deq U(u,\theta) + \eps \, F(\beta,\theta)
\\&= U(u,\theta) + \eps
\,E(\theta,s_\beta(\theta))+\eps\beta\,C(s_\beta(\theta)).
\end{align}
since the state that realizes the minimum of $F$ is $s_\beta(\theta)$.

Let $\theta_\beta^\eps(u)$ be the equilibrium parameter, i.e. the minimizer of the global energy:
\begin{equation}
    \theta_\beta^\eps(u) \deq \underset{\theta}{\arg \min} \; G_\beta^\eps(u,\theta).
\end{equation}
Given a parameter value $\theta$, we denote $u=u_\beta^\eps(\theta)$ the value of the control knobs such that $\theta$ is at equilibrium given $\beta$ and $\eps$, i.e.
\begin{equation}
    u = u_\beta^\eps(\theta) \qquad \Longleftrightarrow \qquad \theta = \theta_\beta^\eps(u).
\end{equation}
Finally we introduce the symmetric non-negative definite matrix
\begin{equation}
M_\beta^\eps(\theta) \deq \partial^2_\theta
G_\beta^\eps(u_\beta^\eps(\theta),\theta).
\end{equation}
(Here $\partial_\theta
G$ and $\partial^2_\theta G$ denote partial derivatives of $G$ with
respect to its second variable, and do \emph{not} include differentiation
of $u(\theta)$ with respect to $\theta$.)
In particular
\begin{equation}
    M_0^0(\theta) = \partial_\theta^2 U(u_0^0(\theta),\theta).
\end{equation}

With this notation, the quantities of
Theorem~\ref{thm:riemannian-sgd}
rewrite as 
\begin{equation}
\label{eq:changeofnotation}
u_t=u^\eps_{\beta_1}(\theta_{t-1}),
\qquad
\theta_{t-1}=\theta_{\beta_1}^\eps(u_t),\qquad
\theta_t=\theta^\eps_{\beta_2}(u_t).
\end{equation}

\begin{rem}
\label{rem:beta1}
Without loss of generality, we can assume that $\beta_1=0$, just by
replacing the energy $E$ with
\begin{equation}
E'(\theta,s)\deq E(\theta,s)+\beta_1 C(s)
\end{equation}
and applying the results to $E'$. This shifts all values of $\beta$ by
$\beta_1$.
\end{rem}

This will be used in some proofs below to use $0$ and $\beta$ instead of $\beta_1$ and
$\beta_2$.

\subsection{Technical Assumption: Smooth, Strict Energy Minimizers}
\label{sec:techhyp}

Here we state the technical assumptions for our formal computations to be
valid: namely, smoothness of all functions involved, and existence,
local uniqueness, and smoothness of the various minimizers.

We also provide a simple sufficient condition (Proposition~\ref{prop:suffhyp}) for this to hold in some neighborhood of the current parameter.

\begin{defi}[Strict minimum]
\label{def:strictmin}
We say that a value $x$ \emph{achieves a strict minimum} of a smooth
function $f$ if $x=\argmin_x f(x)$ and moreover $\partial^2 f(x)/\partial
x^2>0$ at this minimum (in the sense of positive definite matrices for
vector-valued $x$). We say that this holds \emph{locally} if the argmin is
restricted to some neighborhood of $x$.
\end{defi}

\begin{hyp}[Smooth, strict minimizers]
\label{hyp:tech}
We assume that $E$, $C$ and $U$ are smooth functions. We assume that $C$
is bounded below.

Let $\theta_0$ be a parameter value.
We assume that
there exists domains $\Theta\subset \R^{\dim(\theta)}$ in parameter
space, $\mathcal{S}\subset \R^{\dim(s)}$ in state
space, $\mathcal{U}\subset \R^{\dim(u)}$ in control knob space, and open intervals $I_1\subset \R$ and
$I_2\subset\R$ containing $0$, such
that:
\begin{itemize}
\item For any $\theta\in \Theta$ and $\beta\in I_1$, there exists $s_\beta(\theta)\in \mathcal{S}$ which
achieves the strict minimum
\begin{equation}
s_\beta(\theta)=\argmin_{s\in\mathcal{S}} \{E(\theta,s)+\beta\,C(s)\}
\end{equation}
and moreover the map $(\beta,\theta)\mapsto s_\beta(\theta)$ is smooth.
\item
For any $u\in \mathcal{U}$, $\eps\in I_2$, and $\beta\in I_1$, there exists
$\theta^\eps_\beta(u)$ which is the strict minimum
\begin{equation}
\theta^\eps_\beta(u)=\argmin_{\theta\in \Theta} \min_{s\in \mathcal{S}}\left\{
U(u,\theta)+\eps(
E(\theta,s)+\beta C(s))
\right\}
\end{equation}
and the map $(u,\eps,\beta)\mapsto \theta^\eps_\beta(u)$ is smooth.
\item For any $\eps\in I_2$, there exists $u^\eps\in \mathcal{U}$ such that
$\theta^\eps_0(u^\eps)=\theta_0$, namely, when $\beta=0$ we can use $u$ to fix $\theta$
to $\theta_0$.
Moreover, we assume that the map $\eps\mapsto u^\eps$ is smooth.
\end{itemize}
\end{hyp} 

All subsequent values of $\eps$ and $\beta$ will be restricted to $I_2$ and
$I_1$. All subsequent minimizations over $(\theta,s)$ will be taken in
$\Theta\times \mathcal{S}$. Thus, the case where $\Theta\times \mathcal{S}$ is not
the full space allows us, if needed, to
consider only the equilibrium points ``in
the same basin'' as $\theta_0$. Presumably, this is relevant for
\algoname, as a physical system will only jump to another distant
local minimum if it has to.

These assumptions justify the various derivatives and Taylor expansions
in the proofs.

\begin{prop}[A sufficient condition for Assumption~\ref{hyp:tech}]
\label{prop:suffhyp}
Assume that $E$, $C$, and $U$ are smooth, with $C$ bounded below. Let $\theta_0$ be a parameter
value.

Assume that there exists a state $s_0$ that locally achieves a strict minimum of
$E(\theta_0,s_0)$.

Also assume that there exists $u_0$ such that
$\theta_0$ locally achieves a strict minimum of $U(u_0,\theta_0)$. Assume
moreover that $\dim(u)=\dim(\theta)$ and that the matrix $\partial_u \partial_\theta U(u_0,\theta_0)$ is invertible (local controllability of $\theta$ by
$u$).

Then Assumption~\ref{hyp:tech} holds in a domain that contains a
neighborhood of
$(\theta_0, s_0, u_0, \eps=0, \beta=0)$.
\end{prop}

The controllability condition is obviously satisfied for
$U(u,\theta)=\norm{u-\theta}^2$.

\begin{proof}[Proof of Proposition~\ref{prop:suffhyp}]
Let us first check the existence of the smooth minimizer
$s_\beta(\theta)$. Since $E$ and $C$ are smooth, the function
$(\theta,\beta,s)\mapsto E(\theta,s)+\beta \, C(s)
$
is smooth, and therefore, so is the function
\begin{equation}
(\theta,\beta,s)\mapsto f(\theta,\beta,s)\deq \nabla_s (E(\theta,s)+\beta C(s)).
\end{equation}
Moreover, for $\theta=\theta_0$ and $\beta=0$, since $s_0$ is a strict
minimizer of $E(\theta_0,s_0)$, we have that $\nabla_s
f(\theta_0,0,s_0)=\nabla^2_s E(\theta_0,s_0)$ is positive definite.

Therefore, by the implicit function theorem, there exists a smooth map
$(\theta,\beta)\mapsto s_\beta(\theta)$ such that
$f(\theta,\beta,s_\beta(\theta))=0$ in some neighborhood of
$\theta=\theta_0$ and $\beta=0$. By definition of $f$, such an
$s_\beta(\theta)$ is a critical point of $E(\theta,s)+\beta \, C(s)$. This
critical point is a strict minimum: indeed, at $(\theta_0,s_0,\beta=0)$ the second
derivative with respect to $s$ is positive definite, and by continuity
this extends to a neighborhood of $\theta_0$, $s_0$, and $\beta=0$.

For the argmin
\begin{equation}
\theta^\eps_\beta(u)=\argmin_{\theta\in \Theta} \min_{s\in
\mathcal{S}}\left\{
U(u,\theta)+\eps(
E(\theta,s)+\beta C(s))
\right\},
\end{equation}
following the previous proof, we can set the value $s=s_\beta(\theta)$. Therefore this is equivalent to the argmin
\begin{equation}
\theta^\eps_\beta(u)=\argmin_{\theta\in \Theta}
\left\{
U(u,\theta)+\eps( E(\theta,s_\beta(\theta))+\beta C(s_\beta(\theta)))
\right\}
\end{equation}
and since $s_\beta(\theta)$ is smooth,
this quantity is smooth as a function of $u$, $\theta$, and $\beta$. So
once more we can apply the implicit function theorem in a neighborhood of
$u=u_0$, $\eps=0$, $\beta=0$, using that $\theta_0$ is a strict minimum
of $U(u_0,\theta_0)$.

Last, with $\beta=0$, we want to find $u^\eps$ such that
\begin{equation}
\theta_0=\argmin_\theta \min_s \{U(u^\eps,\theta)+\eps E(\theta,s)\}.
\end{equation}
Once more, we can substitute $s=s_0(\theta)$ so we want
\begin{equation}
\theta_0=\argmin_\theta \{U(u^\eps,\theta)+\eps E(\theta,s_0(\theta))\}.
\end{equation}
Set
\begin{equation}
f(\eps,u)\deq  \nabla_\theta \left(
U(u,\theta_0) +\eps E(\theta_0,s_0(\theta_0))\right).
\end{equation}
This is a smooth function.
We have $f(0,u_0)=0$ because $\theta_0$ is a minimizer of
$U(u_0,\theta_0)$. Moreover, we have $\nabla_u f(0,u_0)=\nabla_u
\nabla_\theta U(u_0,\theta_0)$ which is invertible by assumption.
Therefore, by the implicit function theorem, we can find a smooth map
$\eps\mapsto u^\eps$ such that $f(\eps,u^\eps)=0$, namely, such that
$\theta_0$ is a critical point of $U(u^\eps,\theta)+\eps
E(\theta,s_0(\theta))$. This critical point is a strict minimum: indeed,
by assumption this holds for $\eps=0$, and by continuity the second
derivative will stay positive in a neighborhood of $0$.
\end{proof}

\subsection{Relationships Between the Loss, Lyapunov Function, and Energy}
\label{sec:formula-loss-lyapunov}

\begin{thm}[Formulae for the loss and Lyapunov functions]
\label{thm:formula-loss-lyapunov}
We have the following expression for the derivative of $F$ with respect to $\beta$:
\begin{equation}
    \partial_\beta F(\beta,\theta) = C(s_\beta(\theta)).
\end{equation}
Furthermore, the loss function $\mathcal{L}$ and the Lyapunov function $\mathcal{L}_{\beta_1;\beta_2}$ can be expressed in terms of $F$ as
\begin{equation}
    \mathcal{L}(\theta) = \partial_\beta F(0,\theta), \qquad \mathcal{L}_{\beta_1;\beta_2} \left( \theta \right) = \frac{F(\beta_2,\theta) - F(\beta_1,\theta)}{\beta_2-\beta_1}.
\end{equation}
\end{thm}

\begin{proof}[Proof of Theorem \ref{thm:formula-loss-lyapunov}]
First, we note that
\begin{equation}
    \label{lma:min-energy}
    F(\beta,\theta) = \overline{F}(\beta,\theta,s_\beta(\theta)),
\end{equation}
where
\begin{equation}
    \overline{F}(\beta,\theta,s) \deq E(\theta,s) + \beta \, C(s),
\end{equation}
and by construction
\begin{equation}
    s_\beta(\theta) = \underset{s}{\arg \min} \;  \overline{F}(\beta,\theta,s).
\end{equation}
Then we differentiate both sides of \eqref{lma:min-energy} with respect to $\beta$. By the chain rule of differentiation, we have
\begin{equation}
    \label{eq:lma-diff-beta}
    \frac{\partial F}{\partial \beta}(\beta,\theta) = \frac{\partial \overline{F}}{\partial \beta}(\beta,\theta,s_\beta(\theta)) + \frac{\partial \overline{F}}{\partial s}(\beta,\theta,s_\beta(\theta)) \cdot \frac{\partial s_\beta}{\partial \beta}(\theta).
\end{equation}
The first term on the right-hand side of \eqref{eq:lma-diff-beta} is equal to $C(s_\beta(\theta))$ by definition of $\overline{F}$. The second term vanishes since $\frac{\partial \overline{F}}{\partial s}(\beta,\theta,s_\beta(\theta)) = 0$ at equilibrium. Therefore
\begin{equation}
    \label{eq:eqprop-property}
    \frac{\partial F}{\partial \beta}(\beta,\theta) = C(s_\beta(\theta)).
\end{equation}
Evaluating \eqref{eq:eqprop-property} at the point $\beta=0$, and using the definition of $\lyap$, we get
\begin{equation}
    \frac{\partial F}{\partial \beta}(0,\theta) = C(s_0(\theta)) = \mathcal{L}(\theta).
\end{equation}
Furthermore, integrating both hands of \eqref{eq:eqprop-property} from $\beta^\prime=\beta_1$ to $\beta^\prime = \beta_2$, we get
\begin{equation}
    F(\beta_2,\theta) - F(\beta_1,\theta) = \int_{\beta_1}^{\beta_2}
    C(s_{\beta^\prime}(\theta)) \d\beta^\prime.
\end{equation}
Dividing both sides by $\beta_2-\beta_1$ and using the definition of $\mathcal{L}_{\beta_1;\beta_2}$, we conclude that
\begin{equation}
\frac{F(\beta_2,\theta) - F(\beta_1,\theta)}{\beta_2-\beta_1} = \mathcal{L}_{\beta_1;\beta_2} \left( \theta \right).
\end{equation}
\end{proof}

Now we turn to the properties of the Lyapunov function stated in Theorem~\ref{thm:riemannian-sgd}.

\begin{prop}
\label{prop:taylor-expansions}
As $\beta_1, \beta_2 \to 0$, we have the Taylor expansion
\begin{equation}
    \mathcal{L}_{\beta_1;\beta_2}(\theta) = \mathcal{L}(\theta) + O(|\beta_1| + |\beta_2|),
\end{equation}
and for $\beta_2 = -\beta_1 = \beta$, we have
\begin{equation}
    \mathcal{L}_{-\beta/2;\beta/2}(\theta) = \mathcal{L}(\theta) + O(\beta^2).
\end{equation}
\end{prop}

These Taylor expansions are direct consequences of the definition of $\mathcal{L}_{\beta_1;\beta_2}(\theta)$. Alternatively, they can be derived from Theorem \ref{thm:formula-loss-lyapunov}, as follows.

\begin{proof}[Proof of Proposition~\ref{prop:taylor-expansions}]
We write
\begin{align}
    F(\beta_2, \theta) & = F(0,\theta) + \beta_2 \partial_\beta F(0,\theta) + \beta_2^2 \partial_\beta^2 F(0,\theta) + O(\beta_2^3), \\
    F(\beta_1, \theta) & = F(0,\theta) + \beta_1 \partial_\beta F(0,\theta) + \beta_1^2 \partial_\beta^2 F(0,\theta) + O(\beta_1^3),
\end{align}
so that
\begin{equation}
    \frac{F(\beta_2,\theta) - F(\beta_1,\theta)}{\beta_2-\beta_1} = \partial_\beta F(0,\theta) + (\beta_1 + \beta_2) \partial_\beta^2 F(0,\theta) + O(\beta_1^2 + \beta_2^2).
\end{equation}
Using Theorem \ref{thm:formula-loss-lyapunov}, we get
\begin{equation}
    \mathcal{L}_{\beta_1;\beta_2}(\theta) = \mathcal{L}(\theta) + (\beta_1 + \beta_2) \partial_\beta^2 F(0,\theta) + O(\beta_1^2 + \beta_2^2).
\end{equation}
\end{proof}

Next we prove that $\lyap_{0;\beta_2}(\theta)$ and $\lyap_{-\beta_1;0}(\theta)$ are lower and upper bounds of $\mathcal{L}(\theta)$. First, we need a further lemma.

\begin{lma}
\label{lma:monotonous}
Let $\theta$ be any value. For each $\beta\in \R$, let $s_\beta\in
\argmin_s \{E(\theta,s)+\beta C(s)\}$. Then $\beta\mapsto C(s_\beta)$ is
non-increasing.
\end{lma}

We note that this also implies that the function $\beta \to F(\beta,\theta)$ is concave, thanks to the first formula of Theorem \ref{thm:formula-loss-lyapunov}.

\begin{proof}[Proof of Lemma~\ref{lma:monotonous}]
Let $\beta\geq \beta'$.
%Since $s_{\beta}$ minimizes $E(\theta,\cdot) + \beta C(\cdot)$ by definition, we have
By definition of $s_{\beta}$,
\begin{equation}
	E(\theta,s_\beta) + \beta C(s_\beta) \leq E(\theta,s_{\beta^\prime}) + \beta C(s_{\beta^\prime}).
\end{equation}
%Similarly, since $s_{\beta^\prime}$ minimizes $E(\theta,\cdot) + \beta^\prime C(\cdot)$ by definition, we have
Similarly, by definition of $s_{\beta^\prime}$,
\begin{equation}
	E(\theta,s_{\beta^\prime}) + \beta^\prime C(s_{\beta^\prime}) \leq E(\theta,s_\beta) + \beta^\prime C(s_\beta).
\end{equation}
Summing these two inequalities, subtracting $(E(\theta,s_\beta) + E(\theta,s_{\beta^\prime}))$ on each side, and rearranging the terms, we get
\begin{equation}
(\beta-\beta^\prime) C(s_\beta)\leq (\beta-\beta^\prime)C(s_{\beta^\prime}),
\end{equation}
 which proves the claim.
\end{proof}

%Here we assume that the infima over $s$ are minima, and that $C$ is bounded below.
%The following lemmas are standard results on parameterized minimization.

%\begin{prop}
%The function $\beta \mapsto F(\beta,\theta)$ is concave. 
%\end{prop}

%\begin{proof}
%Indeed, for any $\beta$ and $\beta'$,
%\begin{align}
%F\left(
%\frac{\beta+\beta'}{2},\theta
%\right)
%&=\inf_s \left\{E(\theta,s)+\frac{\beta+\beta'}{2} C(s)\right\}
%\\&=\inf_s \left\{
%\frac12 (E(\theta,s)+\beta C(s))+\frac12 (E(\theta,s)+\beta' C(s))
%\right\}
%\\&\geq
%\frac12 \inf_s \{E(\theta,s)+\beta C(s)\}+\frac12 \inf_s
%\{E(\theta,s)+\beta' C(s)\}
%\\&=\frac12 F(\beta,\theta)+\frac12 F(\beta',\theta)
%\end{align}
%\end{proof}

Since $\lyap_{\beta_1;\beta_2}(\theta)$ is the average of
$C(s_\beta(\theta))$ for $\beta\in [\beta_1;\beta_2]$, 
this immediately implies the following.

\begin{cor}
\label{cor:bounds}
For any $\beta_1<0<\beta_2$, we have
\begin{equation}
\lyap_{0;\beta_2}(\theta)\leq \mathcal{L}(\theta)\leq
\lyap_{-\beta_1;0}(\theta).
\end{equation}
\end{cor}

\subsection{Proof of the Lyapunov Property (Theorem~\ref{thm:lyapunov})}
\label{sec:proof-lyapunov}

We now prove monotonous improvement of the Lyapunov function, as stated
in Theorems~\ref{thm:lyapunov} and~\ref{thm:riemannian-sgd}.
%With the notations of Section \ref{sec:notation}, these statements can be rewritten as follows.
%As explained before, we can assume that $\beta_1=0$ and $\beta_2 = \beta > 0$.
Let $\beta_2 > \beta_1$.
Let $u_t$ and $\eps>0$ be fixed. We denote
$U(\theta)$ and $\theta_\beta$ in place of $U(u_t,\theta)$ and
$\theta^\eps_\beta(u_t)$, for simplicity. Then $\theta_{\beta_1}$ is the value of $\theta$ before the
update, and $\theta_{\beta_2}$ its value after the update. 
We claim that
\begin{equation}
    \lyap_{\beta_1;\beta_2}(\theta_{\beta_2}) \leq \lyap_{\beta_1;\beta_2}(\theta_{\beta_1}).
\end{equation}

Indeed, since $\theta_{\beta_2}$ minimizes $G_{\beta_2}(\cdot) = U(\cdot)/\eps + F(\beta_2,\cdot)$ by definition, we have
\begin{equation}
	U(\theta_{\beta_2})/\eps + F(\beta_2,\theta_{\beta_2}) \leq U(\theta_{\beta_1})/\eps + F(\beta_2,\theta_{\beta_1}).
\end{equation}
Similarly, since $\theta_{\beta_1}$ minimizes $G_{\beta_1}(\cdot) = U(\cdot)/\eps + F(\beta_1,\cdot)$ by definition, we have
\begin{equation}
	U(\theta_{\beta_1})/\eps + F(\beta_1,\theta_{\beta_1}) \leq U(\theta_{\beta_2})/\eps + F(\beta_1,\theta_{\beta_2}).
\end{equation}
Summing these two inequalities, subtracting $(U(\theta_{\beta_1}) + U(\theta_{\beta_2}))/\eps$ on each side, rearranging the terms, and dividing by $\beta_2-\beta_1$ (which is positive), we get
\begin{equation}
\frac{F(\beta_2,\theta_{\beta_2}) - F(\beta_1,\theta_{\beta_2})}{\beta_2-\beta_1} \leq \frac{F(\beta_2,\theta_{\beta_1}) - F(\beta_1,\theta_{\beta_1})}{\beta_2-\beta_1}.
\end{equation}
We conclude using Theorem \ref{thm:formula-loss-lyapunov}.

\subsection{A Technical Lemma}
\label{sec:technical}

We now prove a technical lemma that we will use to prove the Riemannian SGD property (Theorem \ref{thm:riemannian-sgd}). 

By Remark~\ref{rem:beta1}, we can assume that $\beta_1=0$, and we just
denote $\beta_2$ by $\beta$.

\begin{lma}[Technical Lemma]
\label{lma:technical}
For any $u\in \mathcal{U}$,
we have $\theta_\beta^\eps(u)-\theta_0^\eps(u) = O(\eps \beta)$ when either $\eps\to 0$ or $\beta\to 0$ (or both).
\end{lma}

\begin{proof}[Proof of Lemma \ref{lma:technical}]
Under Assumption~\ref{hyp:tech}, $(\eps,\beta)\mapsto
\theta^\eps_\beta(u)$ is smooth. Therefore, when $\beta\to 0$ we have
$\theta_\beta^\eps(u)-\theta_0^\eps(u)=O(\beta)$, and when $\eps\to 0$ we
have $\theta_\beta^\eps(u)-\theta_0^\eps(u)=O(\eps)$.

Thus, the only
remaining case is when both $\eps$ and $\beta$ tend to $0$: we have to
establish that $\theta_\beta^\eps(u)-\theta_0^\eps(u)=O(\eps\beta)$.
Since we know that this difference is both $O(\beta)$ and $O(\eps)$, we
already know that this difference tends to $0$.

Since $u$ is fixed, we further simplify notation by omitting $u$.

Under Assumption~\ref{hyp:tech}, $\theta^0_0$ achieves a strict minimum of
$G_0^0$. Therefore, by Definition~\ref{def:strictmin}, the Hessian of
$G_0^0$ at $\theta_0^0$ is positive definite, and there exists $\eta>0$
such that
$
\partial^2_\theta G_0^0(\theta_0^0) \geq \eta \Id
$
in the sense of positive definite matrices. Since $G$ is smooth, by
continuity we can assume that
\begin{equation}
\partial^2_\theta G_\beta^\eps (\theta)\geq \eta \Id/2
\end{equation}
when $\eps$ and $\beta$ are close to $0$ and $\theta$ is in a
neighborhood of $\theta^0_0$.

Now, $\theta_0^\eps$ minimizes $G_0^\eps$. Therefore, for any $\theta$ in
a neighborhood of $\theta_0^0$ we have
\begin{equation}
G_0^\eps(\theta)\geq G_0^\eps(\theta_0^\eps)+\eta
\norm{\theta-\theta_0^\eps}^2/4
\end{equation}
using that the Hessian of $G_0^\eps$ is at least $\eta \Id/2$. In
particular, taking $\theta=\theta_\beta^\eps$,
\begin{equation}
G_0^\eps(\theta_0^\eps)+\eta
\norm{\theta_\beta^\eps-\theta_0^\eps}^2/4
\leq G_0^\eps(\theta_\beta^\eps).
\end{equation}
In turn,
\begin{align}
G^\eps_0(\theta^\eps_\beta) & =U(\theta^\eps_\beta)+\eps \inf_s
E(\theta^\eps_\beta,s)
\\&\leq U(\theta^\eps_\beta)+\eps\inf_s \{E(\theta^\eps_\beta,s)+\beta(C(s)-\inf
C)\}
\\&=U(\theta^\eps_\beta)+\eps\inf_s \{E(\theta^\eps_\beta,s)+\beta C(s)\}
-\eps\beta \inf C
\\&=G^\eps_\beta(\theta^\eps_\beta)-\eps\beta \inf C.
\end{align}
Since $\theta^\eps_\beta$ minimizes $G^\eps_\beta(\theta)$, we have
\begin{align}
G^\eps_\beta(\theta^\eps_\beta) &\leq G^\eps_\beta(\theta^\eps_0)
\\&=U(\theta^\eps_0)+\eps\inf_s
\{E(\theta^\eps_0,s)+\beta C(s)\}
\\&\leq U(\theta^\eps_0)+\eps E(\theta^\eps_0,s_0^\eps)+\eps\beta
C(s_0^\eps)
\\&= G_0^\eps(\theta^\eps_0)+\eps\beta C(s_0^\eps)
\end{align}
where $s_0^\eps$ is the value that realizes the infimum $E(\theta^\eps_0,s)$.
Combining the three inequalities, we find
\begin{equation}
\eta
\norm{\theta^\eps_\beta-\theta^\eps_0}^2/4
\leq \eps\beta (C(s_0^\eps)-\inf C).
\end{equation}
When $\eps\beta\to 0$, $s_0^\eps$ tends to $s_0^0$ so that $C(s_0^\eps)$
is bounded. This implies that
$\theta^\eps_\beta-\theta^\eps_0=O(\sqrt{\eps\beta})$.

Now, since $\theta^\eps_\beta$ minimizes $G^\eps_\beta$, we have $\partial_\theta G^\eps_\beta(\theta^\eps_\beta)=0$. Likewise for
$\theta^\eps_0$, we have $\partial_\theta G^\eps_0(\theta^\eps_0)=0$. Subtracting,
\begin{align}
0&=\partial_\theta G^\eps_\beta(\theta^\eps_\beta)-\partial_\theta
G^\eps_0(\theta^\eps_0)
\\&=
\left[ \partial_\theta G^\eps_\beta(\theta^\eps_\beta)-\partial_\theta
G^\eps_0(\theta^\eps_\beta) \right]
+ \left[ \partial_\theta G^\eps_0(\theta^\eps_\beta)-\partial_\theta
G^\eps_0(\theta^\eps_0) \right]
\\&=\eps \left[ \partial_\theta F(\beta,\theta^\eps_\beta)-\partial_\theta
F(0,\theta^\eps_\beta) \right] + \partial^2_\theta G^\eps_0(\theta^\eps_0)\cdot
(\theta^\eps_\beta-\theta^\eps_0)+O(\norm{\theta^\eps_\beta-\theta^\eps_0}^2).
\end{align}
Since $\theta^\eps_\beta-\theta^\eps_0=O(\sqrt{\eps\beta})$, the last $O$ term is
$O(\eps\beta)$.
Since $F$ is smooth, we have $\partial_\theta
F(\beta,\theta^\eps_\beta)-\partial_\theta
F(0,\theta^\eps_\beta)=O(\beta)$ so the first term is $O(\eps\beta)$ as well.
Therefore,
\begin{equation}
\partial^2_\theta G^\eps_0(\theta^\eps_0)\cdot
(\theta^\eps_\beta-\theta^\eps_0)=O(\eps\beta).
\end{equation}
Now the smallest eigenvalue of
$\partial^2_\theta G^\eps_0$ is at least $\eta/2$. Therefore,
$\theta^\eps_\beta-\theta^\eps_0=O(\eps\beta)$ as needed.
\end{proof}

For Theorem~\ref{thm:riemannian-sgd} we are going to use this lemma with
$u=u^\eps$. When $\eps$ is fixed this is a fixed value of $u$. When
$\eps\to 0$ this is not a fixed value of $u$; however, $u^\eps$ tends to
$u^0$ when $\eps\to 0$, and by continuity of all functions involved, the
constant in $O(\eps\beta)$ in the lemma is uniform in a neighborhood of
$u^0$. Therefore, we will be able to apply the lemma to $u^\eps$ when
$\eps\to 0$.

\subsection{Proof of the Riemannian SGD Property (Theorem \ref{thm:riemannian-sgd})}
\label{sec:proof-riemannian-sgd}

We now prove the remaining part of Theorem~\ref{thm:riemannian-sgd}, i.e., the expression for $\theta_t-\theta_{t-1}$. Let us first rephrase it using the notation introduced so far.

By Remark~\ref{rem:beta1}, assume again that $\beta_1=0$ and $\beta_2 = \beta > 0$.

Let $u^\eps\deq u^\eps_0(\theta_{t-1})$. We
denote for simplicity $\theta_\beta^\eps \deq \theta_\beta^\eps(u^\eps)$.
As mentioned in \eqref{eq:changeofnotation} (Section~\ref{sec:notation}),
we have $\theta_{t-1}=\theta_0^\eps$ and $\theta_t=\theta_\beta^\eps$.

Under the assumptions of Section \ref{sec:techhyp}, we claim that when either $\eps$
or $\beta$ (or both) tend to $0$,
\begin{equation}
\label{eq-riemannian-sgd-reformulation}
\theta_\beta^\eps = \theta_0^\eps - \eps\beta \, M_0^\eps(\theta^\eps_0)^{-1}
\partial_\theta \lyap_{0;\beta}(\theta_0^\eps)+O(\eps^2\beta^2)
\end{equation}
where $\lyap_{0;\beta}$ is the Lyapunov function of Section~\ref{sec:formula-loss-lyapunov} and $M_0^\eps$ is the Riemannian matrix given by
\begin{equation}
\label{eq:defM}
M_0^\eps(\theta^\eps_0) \deq \partial^2_\theta
G_0^\eps(u^\eps,\theta^\eps_0) = \partial^2_\theta
U(u^\eps,\theta^\eps_0)
+\eps \partial^2_\theta F(0,\theta^\eps_0).
\end{equation}

\begin{proof}[Proof of Theorem \ref{thm:riemannian-sgd}]
By definition, $\theta_\beta^\eps$ minimizes $G_\beta^\eps(u^\eps,\cdot) = U(u^\eps,\cdot)
+\eps F(\beta,\cdot)$.
The equilibrium condition for $\theta^\eps_\beta$ writes out as
\begin{equation}
\partial_\theta U(u^\eps,\theta^\eps_\beta) + \eps \, \partial_\theta
F(\beta,\theta^\eps_\beta) = 0.
\end{equation}

Let us subtract this equilibrium condition for arbitrary $\beta$ and for $\beta=0$ :
\begin{equation}
\label{eq:diffeqcond}
\left[ \partial_\theta U(u^\eps,\theta^\eps_\beta)-\partial_\theta U(u^\eps,\theta^\eps_0) \right] + \eps \, \left[ \partial_\theta F(\beta,\theta^\eps_\beta) - \partial_\theta F(0,\theta^\eps_0) \right] = 0.
\end{equation}
On the one side we have
\begin{align}
\partial_\theta U(u^\eps,\theta^\eps_\beta)-\partial_\theta
U(u^\eps,\theta^\eps_0) & = \partial^2_\theta U(u^\eps,\theta^\eps_0)\cdot(\theta^\eps_\beta-\theta^\eps_0) + O(\norm{\theta^\eps_\beta-\theta^\eps_0}^2) \\ 
& = \partial^2_\theta U(u^\eps,\theta^\eps_0)\cdot(\theta^\eps_\beta-\theta^\eps_0) + O(\eps^2 \beta^2),
\end{align}
since $\norm{\theta^\eps_\beta-\theta^\eps_0}=O(\eps\beta)$ by Lemma \ref{lma:technical}. On the other side,
\begin{align}
\partial_\theta F(\beta,\theta^\eps_\beta) - \partial_\theta F(0,\theta^\eps_0)
&= \partial_\theta \left(
F(\beta,\theta^\eps_\beta)-F(\beta,\theta^\eps_0) \right) + \partial_\theta \left( F(\beta,\theta^\eps_0) - F(0,\theta^\eps_0) \right)
\intertext{which, by Theorem~\ref{thm:formula-loss-lyapunov}, is}
&=
\partial_\theta \left(
F(\beta,\theta^\eps_\beta)-F(\beta,\theta^\eps_0) \right) + \beta \partial_\theta \lyap_\beta (\theta^\eps_0)
\\&=\beta \partial_\theta \lyap_\beta(\theta^\eps_0) + \partial^2_\theta
F(\beta,\theta^\eps_0) \cdot (\theta^\eps_\beta-\theta^\eps_0) + O(\norm{\theta^\eps_\beta-\theta^\eps_0}^2)
\\&= \beta \partial_\theta \lyap_\beta(\theta^\eps_0) + \partial^2_\theta
F(0,\theta^\eps_0) \cdot (\theta^\eps_\beta-\theta^\eps_0) + O(\beta\norm{\theta^\eps_\beta-\theta^\eps_0} + \norm{\theta^\eps_\beta-\theta^\eps_0}^2)
\\&=\beta \partial_\theta
\lyap_\beta(\theta^\eps_0) + \partial^2_\theta F(0,\theta^\eps_0) \cdot (\theta^\eps_\beta-\theta^\eps_0) + O(\eps\beta^2)
\end{align}
using Lemma \ref{lma:technical} again.

Thus, returning to \eqref{eq:diffeqcond} again, we find
\begin{equation}
\partial^2_\theta
U(u^\eps,\theta^\eps_0)\cdot(\theta^\eps_\beta-\theta^\eps_0) = 
-\eps \left[ \beta \partial_\theta
\lyap_\beta(\theta^\eps_0)
+ \partial^2_\theta
F(0,\theta^\eps_0)\cdot (\theta^\eps_\beta-\theta^\eps_0) \right]
+O(\eps^2\beta^2),
\end{equation}
namely
\begin{equation}
\theta^\eps_\beta-\theta^\eps_0=-\eps\beta
M_0^\eps(\theta^\eps_0)^{-1}
\partial_\theta \lyap_\beta(\theta^\eps_0)+O(\eps^2\beta^2).
\end{equation}
where the Riemannian Matrix $M_0^\eps$ is given by \eqref{eq:defM}.

Note that this Hessian matrix is positive definite, since
$\theta_0^\eps$ achieves a strict minimum of  $G_0^\eps(u^\eps,\cdot)$ by
definition.

Finally, we have $\theta_0^\eps=\theta_{t-1}$.
When $\eps\to 0$, $u^\eps$ tends to $u^0$ and $M_0^\eps$ is
$\partial^2_\theta U(u^0,\theta_{t-1})+O(\eps)$. By definition, $u^0$ is
the value of $u$ such that $\argmin_\theta U(u^0,\theta)=\theta_{t-1}$,
This is the last claim to be proven in Theorem~\ref{thm:riemannian-sgd}.
\end{proof}

\subsection{Proof of the SGD Property with Quadractic Coupling and $\eps, \beta \to 0$ (Theorem~\ref{thm:sgd})}
\label{sec:proof-sgd}

By Proposition~\ref{prop:taylor-expansions}, the Lyapunov function of Optimistic \algoname is
\begin{equation}
    \lyap_\beta = \lyap_{0;\beta}(\theta)=\lyap(\theta)+O(\beta)
\end{equation}
when $\beta\to 0$. Moreover, the Riemannian metric $M_0^\eps$ is
\begin{equation}
    M_0^\eps(\theta) = M_0^0(\theta) + O(\eps)
\end{equation}
when $\eps \to 0$. Injecting these expressions in \eqref{eq-riemannian-sgd-reformulation} (Theorem \ref{thm:riemannian-sgd}), we get
\begin{equation}
\theta_\beta^\eps = \theta_0^\eps - \eps \beta \, M_0^0(\theta_0^\eps)^{-1} \partial_\theta
\lyap(\theta_0^\eps)+O(\eps\beta^2+\eps^2\beta)
\end{equation}
when both $\eps$ and $\beta$ tend to $0$. In particular, using the quadratic control energy $U(u,\theta)=\norm{u-\theta}^2/2$, we have $M_0^0(\theta) = \Id$ and we recover standard SGD.

Similar results hold for Pessimistic \algoname and Centered \algoname, using that $\mathcal{L}_{-\beta;0}(\theta) = \mathcal{L}(\theta) + O(\beta)$ and $\mathcal{L}_{-\beta/2;\beta/2}(\theta) = \mathcal{L}(\theta) + O(\beta^2)$ (Proposition~\ref{prop:taylor-expansions}).

\clearpage
\section{Simulation Details}
\label{sec:simulation-details}

In this section, we provide the implementation details of our numerical simulations of \algoname on Hopfield-like networks (section \ref{sec:numerical-illustration}).

\paragraph{Datasets.}
We perform experiments on the MNIST and FashionMNIST datasets.
%and CIFAR-10.

The MNIST dataset (the `modified' version of the National Institute of Standards and Technology dataset) of handwritten digits is composed of 60,000 training examples and 10,000 test examples \citep{lecun1998gradient}. Each example $x$ in the dataset is a $28 \times 28$ gray-scaled image and comes with a label $y \in \left\{ 0, 1, \ldots, 9 \right\}$ indicating the digit that the image represents.

The Fashion-MNIST dataset \cite{xiao2017fashion} shares the same image size, data format and the sane structure of training and testing splits as MNIST. It comprises a training set of 60,000 images and a test set of
10,000 images. Each example is a $28 \times 28$ grayscale
image from ten categories of fashion products.

%The CIFAR-10 dataset \citep{krizhevsky2009learning} consists of $60,000$ colour images of $32 \times 32$ pixels. These images are split in $10$ classes (each corresponding to an object or animal), with $6,000$ images per class. The training set consists of $50,000$ images and the test set of $10,000$ images.

\paragraph{Energy minimization.} We recall our general strategy to simulate energy minimization: at every step, we pick a variable (layer or parameter) and we `relax' that variable, i.e. we compute analytically the state of that variable that minimizes the energy, given the state of other variables (layers and parameters) fixed. We are able to do that because, when $E$ is the Hopfield energy and $C$ is the squared error, the global energy $\mathcal{E} = ||u-\theta||^2 / 2\epsilon + E + \beta C$ is a quadratic function of each of its variables (layers and parameters). Using this property, we can then alternate relaxation of the layers and parameters until a minimum of the energy is reached.

More specifically, during each phase of energy minimization, we relax the layers one by one, either from the first hidden layer to the output layer (in the `forward' direction), or from the output layer back to the first hidden layer (in the `backward' direction). Relaxing all the layers one after the other (once each), constitutes one `iteration'. We repeat as many iterations as is necessary until convergence is attained. We decide converge using the following criterion: at each iteration, we measure the $L^1$-norm $\|s_{\rm next} - s_{\rm previous}\|$, where $s_{\rm previous}$ is the state of the layers before the iteration, and $s_{\rm next}$ is the state of the layers after the iteration. The convergence criterion is $\|s_{\rm next} - s_{\rm previous}\| < \xi$, where $\xi$ is a given threshold.

The threshold $\xi$ is itself an adaptive threshold $\xi_t$ that we update at each epoch of training $t$. At the beginning of training, we start with $\xi_0 = 10^{-3}$. Then, at each epoch $t$, we proceed as follows: for each mini-batch in the training set, we measure the $L^1$-norm $\|s_\star^{(2)}-s_\star^{(1)}\|$ between the equilibrium state $s_\star^{(1)}$ of the first phase and the equilibrium state $s_\star^{(2)}$ of the second phase, and we compute $\mu_t$, the mean of $\|s_\star^{(2)}-s_\star^{(1)}\|$ over the entire training set during epoch $t$. Then, at the end of epoch $t$, we set the threshold for epoch $t+1$ to $\xi_{t+1} = \min(\xi_t, \gamma \mu_t)$, for some constant $\gamma$. We choose $\gamma = 0.01$ in our simulations.

\paragraph{Training procedure.} We train our networks with optimistic, pessimistic and centered \algoname. At each training step of SGD, we proceed as follows. First we pick a mini-batch of samples in the training set, $x$, and their corresponding labels, $y$. Then we set the nudging to $0$ and we perform a homeostatic phase. This phase allows us in particular to measure the training loss for the current batch, to monitor training. Next, if the training method is either pessimistic or centered \algoname, we set the nudging to either $-\beta$ or $-\beta/2$ respectively, and we perform a new homeostatic phase. Finally, we set the nudging to the second nudging value (which is $0$, $\beta/2$ or $\beta$ depending on the training method) and we perform a phase with clamped control knobs.

At each iteration of inference (homeostatic phase without nudging), we relax the layers from the first hidden layer to the output layer. We choose to do so because in this phase, the source of external signals comes from the input layer. Conversely, during the phases with non-zero nudging (either $-\beta$, $-\beta/2$, $+\beta/2$ or $+\beta$), we relax the layers from the output layer back to the first hidden layer, because the new source of external signals comes from the output layer.
Finally, in the `clamped' phase (with clamped control knobs), the parameters are all relaxed in parallel.

\paragraph{Weight initialization.} We initialize the weights of dense interactions according to (half) the `xavier uniform' scheme, i.e.
\begin{equation}
    \label{eq:xavier_uniform}
    w_{ij} \sim \mathcal{U}(-c,+c), \qquad c = \frac{\alpha}{2} \sqrt{\frac{6}{\text{fan\_in} + \text{fan\_out}}},
\end{equation}
where $\alpha$ is a gain, i.e. a scaling number. See Table \ref{tab:hyperparam-mnist} for the choice of the gains. We initialize the weights of convolutional interactions according to (half) the `kaiming normal' scheme, i.e.
\begin{equation}
    \label{eq:kaiming_normal}
    w_{ij} \sim \mathcal{N}(0,c), \qquad c = \frac{\alpha}{2} \sqrt{\frac{1}{\text{fan\_in}}},
\end{equation}
where $\alpha$ is a gain. The factor $\frac{1}{2}$ in \eqref{eq:xavier_uniform} and \eqref{eq:kaiming_normal} comes from the fact that, unlike feedforward networks where each layer receives input only from the bottom layer, in Hopfield networks, hidden layers receive input from both the bottom layer and the upper layer.

%\paragraph{Soft targets.} We use soft targets, that is, instead of using the usual one-hot code for a label $y$ of the form $[0, \ldots, 0, 1, 0, \ldots, 0]$, we use a code of the form $[0, \ldots, 0, c, 0, \ldots, 0]$ for some constant $c$. We choose $c=0.9$.

\paragraph{Simulation details.}
The code for the simulations uses PyTorch 1.9.0 and TorchVision 0.10.0. \cite{paszke2017automatic}.
The simulations were carried on a server of GPUs. For the dense networks, each run was performed on a single GPU for an average run time of 6 hours. For the convolutional networks, each run was performed on a single GPU for an average run time of 30 hours. The parameters were chosen based on trial and errors (Table~\ref{tab:hyperparam-mnist}).

\paragraph{Benchmark.} We compare the three \algoname training procedures (optimistic, pessimistic and centered) against \textit{automatic differentiation} (autodiff). To establish the benchmark via autodiff, we proceed as follows: we unfold the graph of computations during the free phase minimization (with $\beta=0$), and we compute the gradient with respect to the parameters. We then take one step of gradient descent for each parameter $\theta_k$, with step size $\beta \eps_k$.

\begin{table}
\caption{Hyper-parameters used for the simulations on MNIST and FashionMNIST with Hopfield-like networks.}
\label{tab:hyperparam-mnist}
\vspace{0.5cm}
\centering
\begin{tabular}{ccc}
\hline
Hyper-parameter & Dense Network & Convolutional Network \\
\hline
layer shapes & $1\plh28\plh28 - 2048 - 10$ & $1\plh28\plh28 - 32\plh12\plh12 - 64\plh4\plh4 - 10$ \\
weight shapes & $1\plh28\plh28\plh2048 - 2048\plh10$ & $32\plh1\plh5\plh5 - 64\plh32\plh5\plh5 - 64\plh4\plh4\plh10$ \\
state space ($\mathcal{S}$) & $[0,1]^{2048}\times[-1,2]^{10}$ & $[0,1]^{32\plh12\plh12}\times[0,1]^{64\plh4\plh4}\times[-1,2]^{10}$ \\
gains ($\alpha$) & 0.8 - 1.2 & 0.6 - 0.6 - 1.5 \\
initial threshold ($\xi_0$) & 0.001 & 0.001 \\
max iterations (first phase) & 100 & 100 \\
max iterations (second phase) & 100 & 100 \\
nudging ($\beta$) & 0.5 & 0.2 \\
batch size & 32 & 16 \\
learning rates (weights) & 0.1 - 0.05 & 0.128 - 0.032 - 0.008 \\
learning rates (biases) & 0.02 - 0.01 & 0.032 - 0.008 - 0.002 \\
decay of learning rates & 0.99 & 0.99 \\
\hline
\end{tabular}
\end{table}

\clearpage

\begin{figure}
\begin{center}
\includegraphics[width=\textwidth]{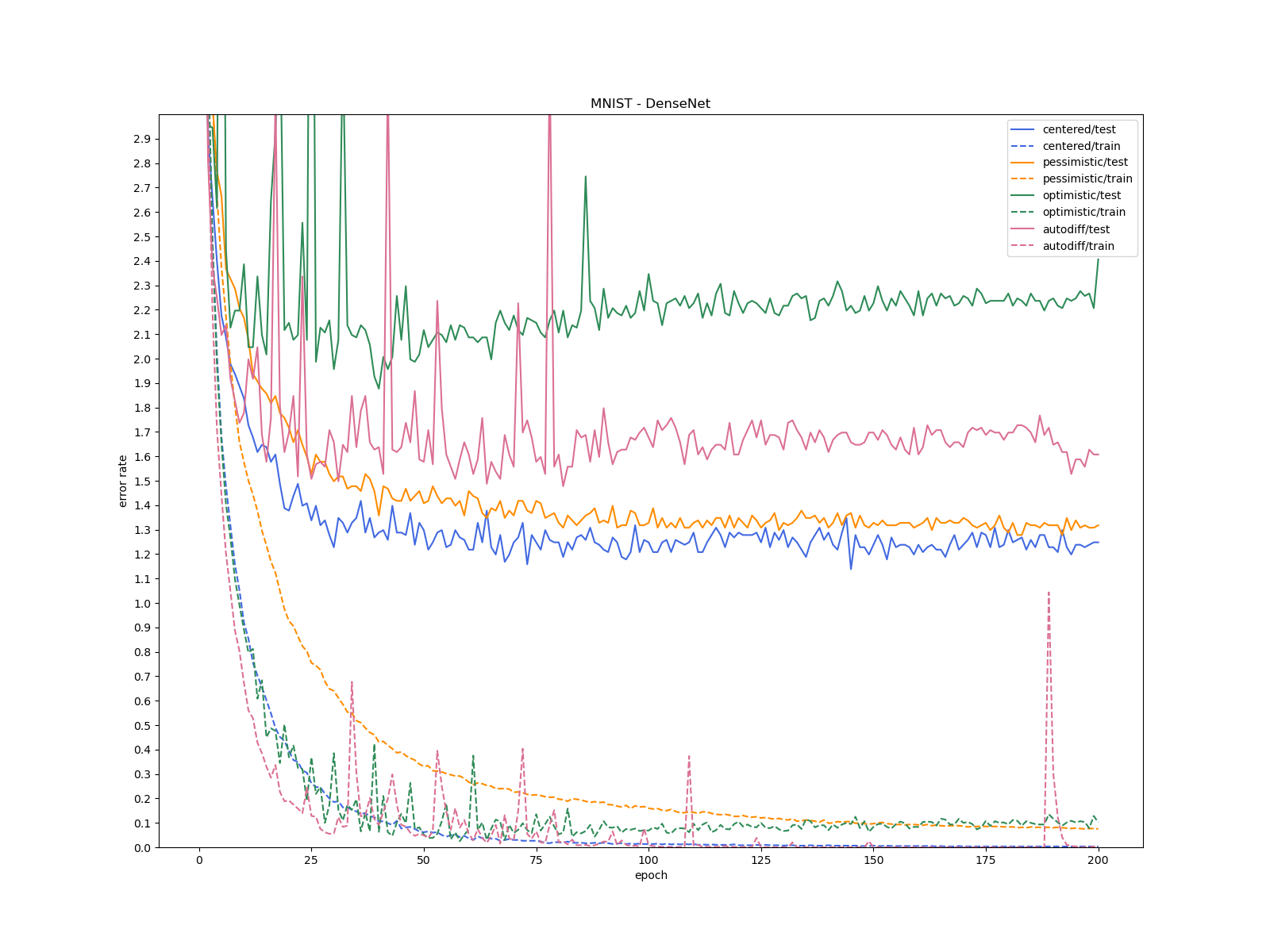}
\includegraphics[width=\textwidth]{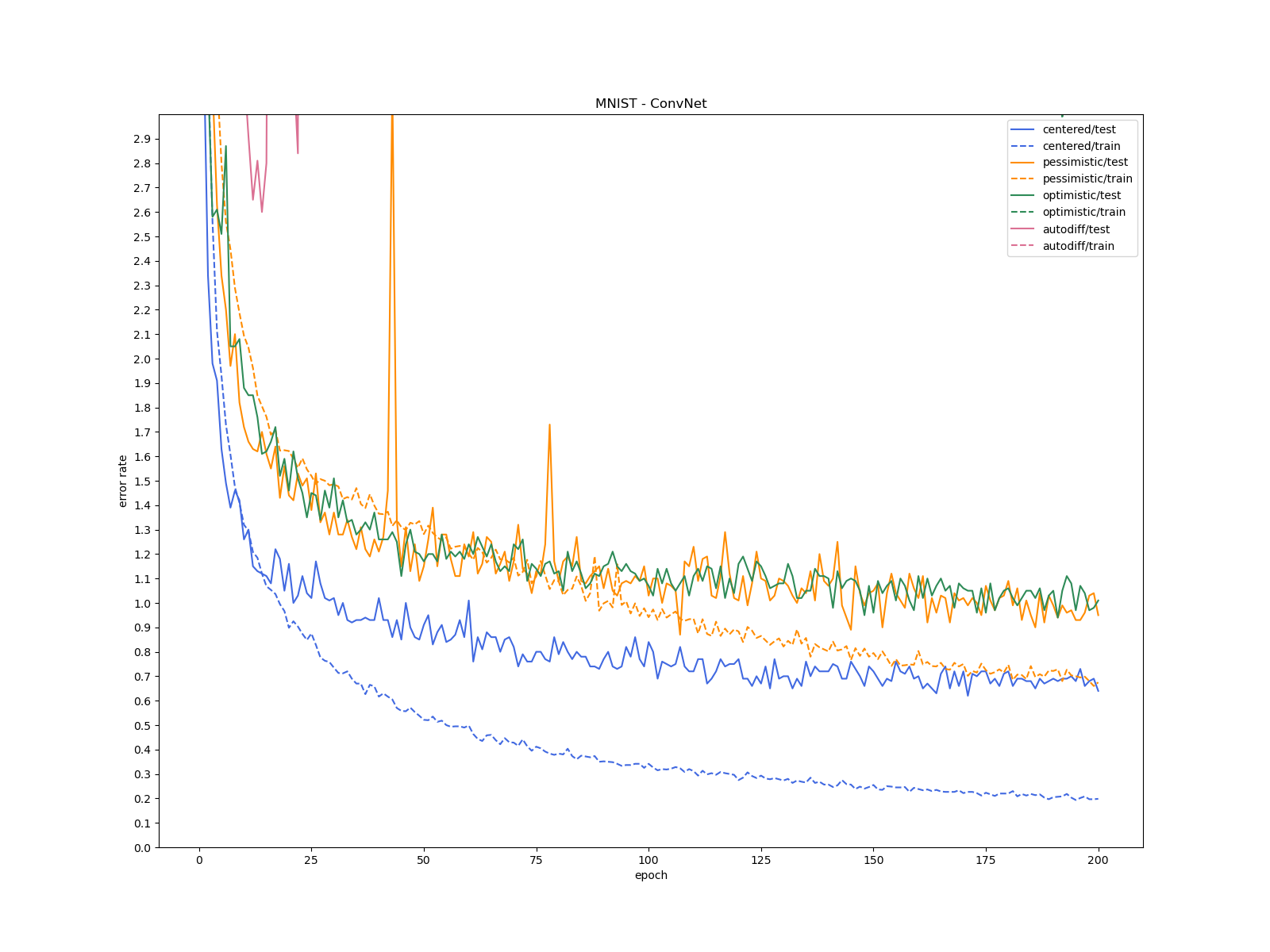}
\end{center}
\caption{Dense and Convolutional Hopfield-like Networks trained via \algoname on MNIST}
\label{fig:mnist}
\end{figure}

\begin{figure}
\begin{center}
\includegraphics[width=\textwidth]{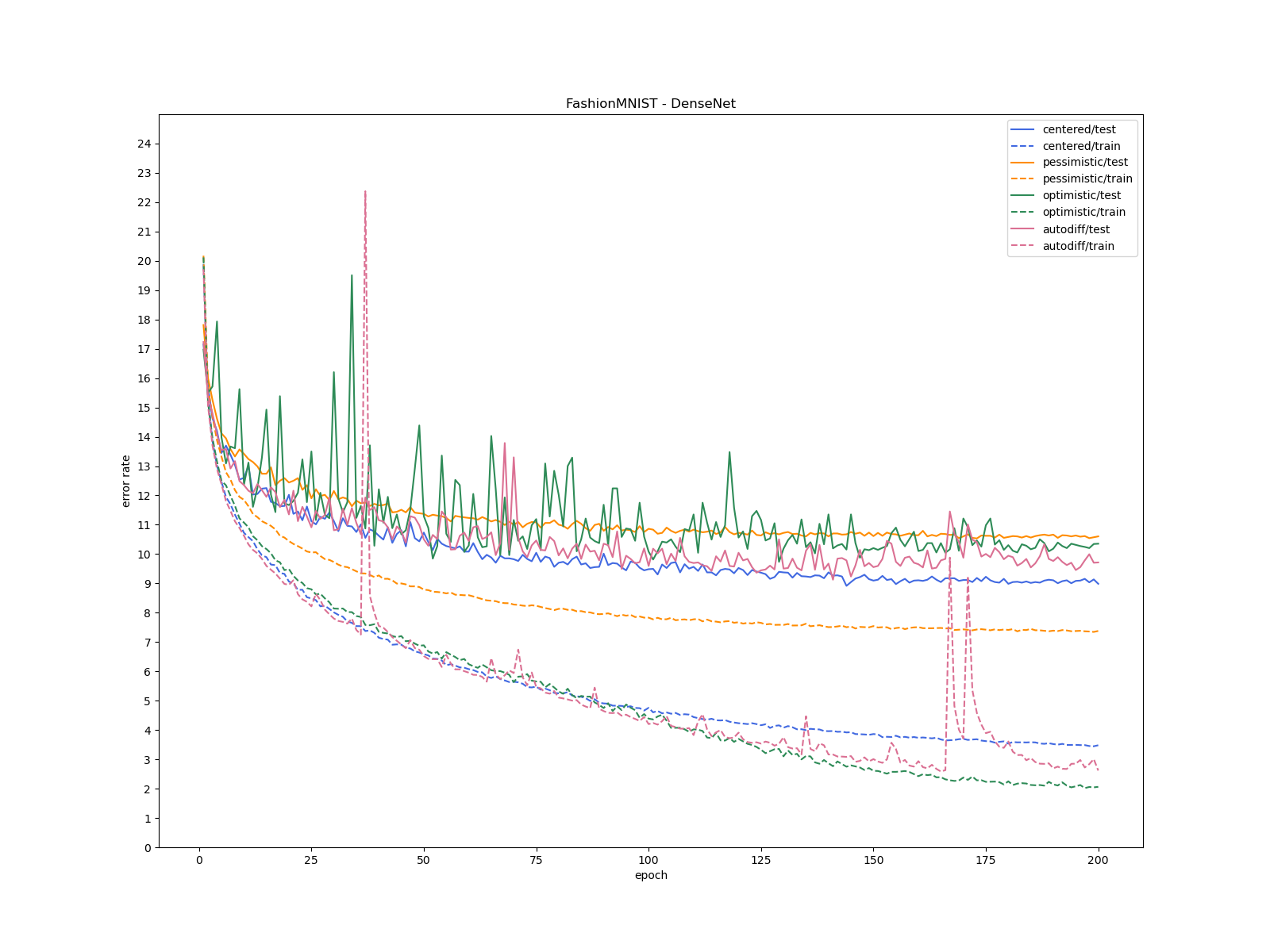}
\includegraphics[width=\textwidth]{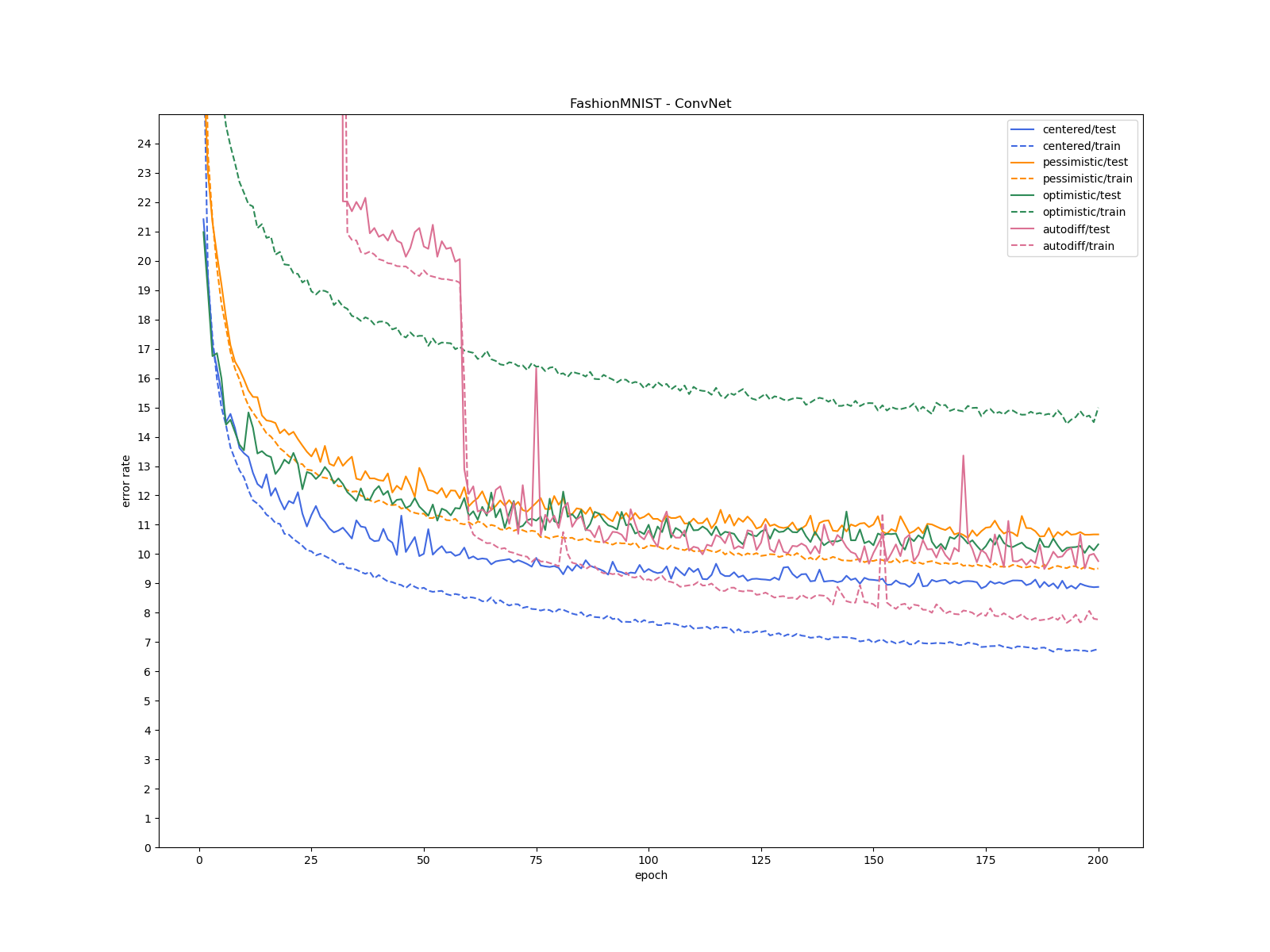}
\end{center}
\caption{Dense and Convolutional Hopfield-like Networks trained via \algoname on FashionMNIST}
\label{fig:fmnist}
\end{figure}

\clearpage
\section{From Eqprop to Agnostic Eqprop}
\label{sec:eqprop}

In this section, we present equilibrium propagation (Eqprop) \cite{scellier2017equilibrium} and explain in more details the problems of Eqprop that Agnostic Eqprop (\algoname) solves.

Recall that we consider an optimization problem of the form
\begin{equation}
    J(\theta) \deq \E_{(x,y)} \, \left[ \lyap(\theta,x,y) \right], \qquad \text{where} \qquad \lyap(\theta,x,y) \deq C(s(\theta,x),y),
\end{equation}
where $C$ is a cost function and $s(\theta,x)$ is a minimizer of some other function $E$:
\begin{equation}
    s(\theta,x) \deq \underset{s}{\arg \min} \; E(\theta,x,s).
\end{equation}
We call $E$ the energy function and $s(\theta,x)$ the equilibrium state.
The idea of Eqprop is to augment the energy at the output part of the system (the $s$-part) by adding an energy term $\beta \, C(s,y)$ proportional to the cost. The total energy of the system is then $E(\theta,x,s) + \beta \, C(s,y)$. As we vary $\beta$, the total energy varies, and therefore the equilibrium state varies, too. Specifically, for every nudging value $\beta$, we define the equilibrium state
\begin{equation}
    \label{eq:eqprop-equilibrium}
    s_\beta \deq \underset{s}{\arg \min} \, \left[  E(\theta,x,s) + \beta \, C(s,y) \right].
\end{equation}
In particular $s_0 = s(\theta,x)$. The main theoretical result of Eqprop is that the loss gradients can be computed by varying the nudging factor $\beta$, via the following formula.

\begin{thm}[Equilibrium propagation]
\label{thm:eqprop}
The gradient of the loss is equal to
\begin{equation}
    \label{eq:static-eqprop}
    \frac{\partial \mathcal{L}}{\partial \theta}(\theta,x,y) = \left. \frac{d}{d\beta} \right|_{\beta=0} \frac{\partial E}{\partial \theta} \left( \theta,x,s_\beta \right).
\end{equation}
\end{thm}

In this expression, $\frac{\partial E}{\partial \theta}$ represents the
partial derivative of $E(\theta,x,s)$ with respect to its first argument, $\theta$ ; we note that $s_\beta$ also depends on $\theta$ through Eq.~\eqref{eq:eqprop-equilibrium}, but importantly,
$\frac{\partial E}{\partial \theta}(\theta,x,s_\beta)$ does not take into account the differentiation paths through
$s_\beta$.
Thanks to Theorem \ref{thm:eqprop}, we can estimate the gradient of $\mathcal{L}$ with finite differences, using e.g. the first-order finite difference forward estimator
\begin{equation}
\label{eq:eqprop-estimator}
\widehat{\nabla}(\beta,\theta,x,y) \deq \frac{1}{\beta} \left( \frac{\partial E}{\partial \theta} \left( \theta, x, s_\beta \right) - \frac{\partial E}{\partial \theta} \left( \theta, x, s_0 \right) \right).
\end{equation}
We note that $\widehat{\nabla}(\beta,\theta,x,y)$ depends on $y$ through $s_\beta$. Eqprop training then consists in optimizing the objective $J(\theta)$ by stochastic gradient descent:
\begin{equation}
    \label{eq:eqprop-update}
    \theta_t \deq \theta_{t-1} - \eta \widehat{\nabla}(\beta,\theta_{t-1},x_t,y_t),
\end{equation}
where, at each step $t$, $\theta_{t-1}$ is the previous parameter value, $(x_t,y_t)$ is an input/target pair taken from the training set, and $\eta$ is the learning rate.
The gradient estimator $\widehat{\nabla}(\beta,\theta_{t-1},x_t,y_t)$ can be obtained with two phases and two measurements, as follows. In the first phase, we present input $x_t$ to the system, we set the nudging factor $\beta$ to zero, and we let the system's state settle to equilibrium, $s_0$. For each parameter $\theta_k$, the quantity $\frac{\partial E}{\partial \theta_k}$ is measured and stored. In the second phase, we present the desired output $y_t$ and set the nudging factor $\beta$ to a positive value, and we let the system settle to a new equilibrium state $s_\beta$. For each parameter $\theta_k$, the quantity $\frac{\partial E}{\partial \theta_k}$ is measured again. Finally, the parameters are updated in proportion to their gradient using \eqref{eq:eqprop-update}.

However, Eqprop training presents several challenges for physical implementations, including the following three. First of all, for each parameter $\theta_k$, the partial derivatives $\frac{\partial E}{\partial \theta_k}$ need to be measured in both phases. To this end, some knowledge about the analytical form of the energy function is necessary, which can be a limitation in physical systems whose components' characteristics are unknown or only partially known.
Second, the quantities $\frac{\partial E}{\partial \theta_k}$ of the first phase need to be stored, since they are no longer physically available at the end of the second phase when the parameters are updated. Third and most importantly, after computing the gradient estimators, we still need to update the parameters according to some (nontrivial) physical procedure. The \algoname method presented in this work fixes these three issues at once.

To derive \algoname from Eqprop, our starting point is Lemma~\ref{lma:monotonous}.
%the following observation.
For brevity of notation, we omit $\theta$, $x$ and $y$, and we denote $s_\beta$ the state that minimizes $E(s)+\beta C(s)$. Using this notation, if $\frac{\partial C}{\partial s} \left( s_0 \right) \neq 0$, then for $\beta>0$ small enough, the perturbed equilibrium state $s_\beta$ yields a lower value of the underlying cost function than $s_0$ i.e., $C(s_\beta) < C(s_0)$. More specifically, we have the following formula for the derivative of $s_\beta$ with respect to $\beta$:
\begin{equation}
\label{eq:eqprop-riemann}
\left. \frac{\partial s_\beta}{\partial \beta} \right|_{\beta=0} = - \frac{\partial^2 E}{\partial s^2} \left( s_0 \right)^{-1} \cdot \frac{\partial C}{\partial s} \left( s_0 \right).
\end{equation}
This is shown by differentiating the equilibrium condition $\partial_s E(s_\beta) + \beta \, \partial_s C(s_\beta) = 0$ with respect to $\beta$. Written as a Taylor expansion, \eqref{eq:eqprop-riemann} rewrites
\begin{equation}
\label{eq:eqprop-taylor}
s_\beta = s_0 - \beta \, \partial_s^2 E \left( s_0 \right)^{-1} \cdot \partial_s C \left( s_0 \right) + O(\beta^2),
\end{equation}
where the Hessian $\partial_s^2 E \left( s_0 \right)$ is positive definite, provided that $s_0$ is a proper minimum of $E(s)$.
The main thrust of \algoname is to establish a formula similar to \eqref{eq:eqprop-taylor} for the parameters, by viewing them as another set of floating variables that minimize the system's energy (like the state variables). The SGD property (Theorem~\ref{thm:sgd}) and the more general Riemannian SGD property (Theorem~\ref{thm:riemannian-sgd}) shown in this paper achieve this.

We note that the formulae of Section \ref{sec:formula-loss-lyapunov} relating the loss, Lyapunov function and energy (Theorem~\ref{thm:formula-loss-lyapunov}, Proposition~\ref{prop:taylor-expansions} and Corollary~\ref{cor:bounds}) hold more broadly in the context of Eqprop. In particular, the gradient \textit{estimator} \eqref{eq:eqprop-estimator} of the true loss $\lyap$ is the \textit{true} gradient of the Lyapunov function $\lyap_\beta$ :
\begin{equation}
\widehat{\nabla}(\beta,\theta) = \frac{1}{\beta} \left( \partial_\theta F(\beta,\theta) - \partial_\theta F(0,\theta) \right) = \partial_\theta \lyap_\beta(\theta),
\end{equation}
where we recall that $F(\beta,\theta) := \min_s (E(\theta,s) + \beta C(s))$.
But in Eqprop, unlike in \algoname, the function $\lyap_\beta$ does not necessarily decrease at each step of training: if the learning rate $\eta$ is too large, $\lyap_\beta$ may increase after one step of \eqref{eq:eqprop-update}, like in standard SGD.

Directly derived from Eqprop is the method proposed by \cite{stern2021supervised} called \textit{coupled learning}. \cite{stern2021supervised} considers the case of the squared error cost function $C(s) = \| s-y \|^2$, for which we have $\partial_s C \left( s \right) = (s-y)$. With this choice of $C$, and assuming that $\partial_s^2 E \approx \text{Id}$, Eq.~\eqref{eq:eqprop-taylor} yields $s_\beta \approx s_0 - \beta \, (s_0-y)$. Thus, to achieve nudging in the second phase, instead of adding an energy term $\beta C(s)$ to the system as in Eqprop, \cite{stern2021supervised} propose to clamp the output unit to the state
\begin{equation}
s_{\rm clamped} := (1-\beta) s_0 + \beta y,
\end{equation}
and to let the system relax to equilibrium.
However, contrary to Eqprop, this method does not in general compute the gradient of the loss, even in the limit of infinitesimal perturbation ($\beta \to 0$), except in the special case where the Hessian of $E$ is the identity matrix.

Theorem \ref{thm:eqprop} also has implications for meta-learning and other bilevel optimization problems: \citep{zucchet2021contrastive} introduced the \textit{contrastive meta-learning} rule (CML), which uses the differentiation method of Eqprop to compute the gradients of the meta-parameters. We refer to \cite{zucchet2022beyond} for a review of implicit gradient methods in bilevel optimization problems.

\end{document}